\documentclass{arXiv}

\usepackage[nomessages]{fp}
\usetikzlibrary{patterns}

\DeclareMathOperator*{\argmax}{arg\,max}

\newcommand{\R}{\mathbb{R}}
\newcommand{\N}{\mathbb{N}}

\newcommand{\PSI}{\mathbf{\Psi}}

\DeclareMathOperator*{\argmin}{arg\,min~}
\let\argmax\relax
\DeclareMathOperator*{\argmax}{arg\,max~}

\usepackage{mathtools}
\usepackage{kbordermatrix}
\kbalignrighttrue

\definecolor{boxback}{gray}{0.95}

\begin{document}

\title{Feature space approximation for kernel-based supervised learning}

\abstract{We propose a method for the approximation of high- or even infinite-dimensional feature vectors, which play an important role in supervised learning. The goal is to reduce the size of the training data, resulting in lower storage consumption and computational complexity. Furthermore, the method can be regarded as a regularization technique, which improves the generalizability of learned target functions. We demonstrate significant improvements in comparison to the computation of data-driven predictions involving the full training data set. The method is applied to classification and regression problems from different application areas such as image recognition, system identification, and oceanographic time series analysis.}

\keywords{supervised learning, kernel-based methods, feature spaces, \linebreak dimensionality reduction, system identification}

\msc{
68Q27, 
68Q32, 
68T09 
}
\doi{}
\author{Patrick Gel\ss}{Department of Mathematics \\and Computer Science, \\Freie Universität Berlin,\\ Berlin 14195, Germany}
\author{Stefan Klus}{Department of Mathematics,\\University of Surrey,\\Guildford GU2 7XH, UK}
\author{Ingmar Schuster}{Zalando Research, \\
Zalando SE Berlin,\\Berlin 10243, Germany}
\author{Christof Schütte}{Zuse Institute Berlin,\\Berlin 14195, Germany /\\ Department of Mathematics \\and Computer Science, \\Freie Universität Berlin,\\ Berlin 14195, Germany}
\email{p.gelss@fu-berlin.de}

\maketitle

\section{Introduction}

Over the last years, supervised learning algorithms have become an indispensable tool in many different scientific fields. Supervised learning -- a subbranch of machine learning -- is the approach to make predictions and decisions based on given input-output pairs, so-called training data. These data sets can represent highly diverse types of information so that supervised learning techniques have been widely applied to real-world problems in, e.g., physics and chemistry~\cite{CARLEO2019, CARTWRIGHT2020}, medicine~\cite{DEO2015, GLASER2019}, and finance~\cite{OZBAYOGLU2020}. Further popular examples are image classification~\cite{KLUS2019, BASHA2020}, spam filtering~\cite{ALGHAMDI2018, DADA2019}, and the identification of governing equations of dynamical systems~\cite{BRUNTON2016, GELSS2019}. Learning algorithms can be subdivided into regression and classification methods. The former are used to estimate the relationship between input and output vectors, while the latter are used to identify the category to which a given input vector belongs. There exist a lot of different supervised learning methods such as (non-)linear regression, support-vector machines, decision trees, and neural networks. For a detailed overview of different algorithms and applications, we refer to~\cite{TUFFERY2011}.

The field is continuously evolving and the demand for machine learning has grown significantly in the last decade. In this work, we will focus on one of the most powerful mathematical tools in this area, namely kernel-based regression and classification techniques. In this context, a kernel is a function that enables us to avoid explicit representations of high-dimensional feature maps which are used as a basis for learning nonlinear target functions. The main advantage of kernel methods is that inner products are not evaluated explicitly, but only implicitly through the kernel. Linear methods that can be entirely written in terms of kernel evaluations can be turned into nonlinear methods by replacing the standard inner product by a kernel function. This is sometimes referred to as \emph{kernelization} and generally leads to increased performance and accuracy. Examples of such methods are kernel PCA~\cite{SCHOELKOPF1998}, kernel TICA~\cite{HARMELING2003}, kernel EDMD~\cite{WILLIAMS2015, KSM19}, and kernel SINDy/MANDy~\cite{KLUS2019}. 

If the training data set contains highly similar entries or if one feature vector can be written as a linear combination of other feature vectors, the resulting transformed data matrices and Gram matrices will be ill-conditioned or even singular, which can cause numerical instabilities. This is typically alleviated by regularization or other techniques like reduced set methods~\cite{SCHOELKOPF2002} and Nyström approximations~\cite{DRINEAS2005}. In this paper, we will propose a set-reduction technique for kernel-based regression and classification methods that extracts important samples from the training data set -- and therefore relevant information encoded in the feature vectors -- in order to reduce the storage consumption as well as the computational costs. Additionally, the method can also be seen as a regularization technique. 

The remainder of the paper is structured as follows: In Section~\ref{sec: Preliminaries}, we give a brief overview of regression and classification problems as well as the use of feature maps and kernel functions for supervised learning methods. In Section~\ref{sec: Data reduction}, the basic assumptions for the feature space approximation are outlined. The data reduction is then considered from an analytical as well as a heuristic point of view, where the latter results in a method called \emph{kernel-based feature space approximation} (kFSA). Numerical results for benchmark problems from different applications areas are presented in Section~\ref{sec: Results}. Section~\ref{sec: Conclusion} concludes with a brief summary and a future outlook.

\section{Preliminaries}\label{sec: Preliminaries}

We begin with a brief overview of regression and classification problems as well as their formulation in terms of feature spaces and kernels. The notation used throughout this paper is summarized in Table~\ref{tab: notation}. At some points though, we will slightly abuse the notation as we do not distinguish between a data matrix $X$ and the set of its columns, which is also simply denoted by $X$.

\begin{table}[h]
\renewcommand{\arraystretch}{1.3}
  \caption{Notation for special matrices, functions, and spaces.}
  \centering
  \begin{tabular}{ll}
    \hline
    \textbf{Symbol} & \textbf{Description} \\
    \hline 
    $\mathcal{S}$ & sample space (subset of $\R^d$) \\
    $\mathcal{H}$ & feature space (subset of $\R^n$, $n \gg d$, or $\ell^2$) \\
    $\Psi \colon \mathcal{S} \to \mathcal{H}$ & feature map \\
    $X = \big[ x^{(1)}, \dots, x^{(m)}\big]$ & input matrix in $\R^{d \times m}$\\
    $Y = \big[ y^{(1)}, \dots, y^{(m)}\big]$ & output matrix in $\R^{d^\prime \times m}$\\
    $\Psi_X$  & transformed data matrix in $\R^{n \times m}$\\
    $k \colon \mathcal{S} \times \mathcal{S} \to \R$ & kernel function\\
    $\kappa$ & kernel parameter in $\R^+$ \\
    $G_{X, X^\prime}$ & Gram matrix in $\R^{m \times m^\prime}$\\
    $\Theta$  & coefficient matrix\\
    $f \colon \mathcal{S} \to \R^{d^\prime}$ & target (regression or decision) function\\
    $E \colon \mathcal{S} \times \R^{d \times m^\prime} \to \R^+_0$ & approximation error\\
    $\varepsilon$ & threshold for feature space approximation in $\R^+_0$\\
    $\gamma$ & regularization parameter in $\R^+$\\
    \hline
  \end{tabular}
  \label{tab: notation}
\end{table}

\subsection{Regression and classification problems}

Regression and classification problems are subsumed under \emph{supervised learning}, i.e., the data-driven approach to construct a function that correctly maps input to output variables. Given a training set of samples in the form of vectors $x^{(j)} \in \R^d$ and corresponding output vectors $y^{(j)} \in \R^{d^\prime}$, with ${j = 1, \dots, m}$, the aim of regression is to model the relationship between both sets by a function ${f \colon \R^d \to \R^{d^\prime}}$ which enables us to predict the output $y$ for any given test vector $x$. Classification problems are here seen as a special case of regression analysis where the output variable is categorical. That is, $y$ determines the label of $x$, where the output vectors are typically given in the \emph{one-hot encoding} format~\cite{POTDAR2017}: 
\begin{equation}
y_i =
\begin{cases}
1, & \text{if } x \text{ belongs to category } i, \\
0, & \text{otherwise},
\end{cases}
\end{equation}
where $d^\prime$ is the number of considered categories. The advantage of one-hot encoding is that the `distance' between any two categories is always the same, thus spurious effects when the (non-ordinal) labels of two dissimilar samples are close to each other are avoided.

In this work, we want to construct regression and decision functions, respectively, as linear combinations of preselected basis functions $\psi_i$, i.e., 
\begin{equation}\label{eq: basis representation}
  f(x) = \sum_{i \in I } \lambda_i \psi_i(x),
\end{equation}
where $I$ is an appropriate index set. In many cases, however, the number of basis functions can be extremely large (or even countably infinite, i.e., $I \cong \N$) causing high numerical costs. Thus, we will exploit kernel functions in order to reduce the storage consumption and computational complexity. For more details on kernel functions and their application to supervised learning tasks, we refer to~\cite{SCHOELKOPF2002} and~\cite{BONACCORSO2017}.

\subsection{Feature maps and kernel functions}\label{sec: tdt and kernels}

If the learning algorithm can be entirely written in terms of kernel evaluations, explicit basis function evaluations are not needed in practice. This is called the \emph{kernel trick} and allows us to reduce the computational complexity significantly since we do not have to compute inner products in the potentially infinite-dimensional space. However, for the following derivations it is crucial to clarify the relationship between (implicitly) given basis functions and the corresponding kernel. Additionally, we will later consider examples where the explicit representation of the regression function in terms of basis functions is of interest. Consider a vector $x \in \R^d$ and a mapping $\Psi \colon \R^d \to \R^n$, typically with $n \gg d$, given by
\begin{equation*}
\Psi(x) := \left[ \psi_1(x), \dots, \psi_n(x) \right]^\top \in \R^{n},
\end{equation*}
where the real-valued functionals $\psi_1, \dots, \psi_n$ define the preselected basis functions. 

\begin{Remark}\label{rem: finite-dimensional}
In what follows, we will focus on finite-dimensional feature maps to elucidate our approach. However, the concepts of kernels, Gram matrices, and pseudoinverses described below can be adapted to infinite-dimensional mappings $\Psi \colon \R^d \to \mathcal{H}$, e.g., $\mathcal{H} \subseteq \ell^2$, if $\Psi$ is a bounded operator with closed range, see~\cite{BENISRAEL2003}.
\end{Remark}

\noindent Given a data set ${X = [x^{(1)} , \dots, x^{(m)}] \in \R^{d \times m}}$, we define the \emph{transformed data matrix}
\begin{equation*}
  \Psi_X := \left[ \Psi\left(x^{(1)}\right), \dots, \Psi\left(x^{(m)}\right) \right] \in \R^{n \times m}.
\end{equation*}
The transformation $\Psi$ can be interpreted as a \emph{feature map} and, thus, defines a so-called \emph{kernel}. 

\begin{Definition}
A function $k \colon \mathcal{S} \times \mathcal{S} \to \R$ on a non-empty set $\mathcal{S}$ is called a kernel if there exist a real Hilbert space $\mathcal{H}$ and a feature map $\Psi \colon \mathcal{S} \to \mathcal{H}$ such that
\begin{equation}\label{eq: kernel}
  k(x, x^\prime) = \langle \Psi(x) , \Psi(x^\prime) \rangle_{\mathcal{H}}
\end{equation}
for all $x, x^\prime \in \mathcal{S}$.
\end{Definition}

By definition, a kernel is a symmetric and positive semidefinite function. The space $\mathcal{H}$ is called the \emph{feature space} of $k$. In what follows, we will generally consider $\mathcal{S}$ as the space of training and test samples, i.e., $\mathcal{S} \subseteq \R^d$. In this regard, we will slightly abuse the notation at some points by simply denoting the column set of a given data matrix $X$ by the same symbol. Furthermore, we consider $\mathcal{H}$ as a subspace of $\R^n$, see Remark~\ref{rem: finite-dimensional}, and use the Euclidean inner product. The \emph{Gram matrix} corresponding to the kernel $k$ and a given data matrix $X$ is defined as
\begin{equation*}
  G_{X,X} = \Psi_X^\top \Psi_X = 
  \begin{bmatrix}
    k(x^{(1)}, x^{(1)}) & \cdots & k(x^{(1)}, x^{(m)})\\
    \vdots & \ddots & \vdots\\
    k(x^{(m)}, x^{(1)}) & \cdots & k(x^{(m)}, x^{(m)})
  \end{bmatrix} \in \R^{m \times m}.
\end{equation*}
In the following sections, we will also consider Gram matrices corresponding to two different data sets $X, X^\prime$. These are then denoted by $G_{X, X^\prime} = \Psi_X^\top \Psi_{X^\prime}$. 

\begin{Remark}
  The feature space representation~\eqref{eq: kernel} is in general not unique. That is, inner products of different feature maps can result in the same kernel. On the other hand, Mercer's theorem~\cite{MERCER1909} states that we can find a (potentially infinite-dimensional) feature space representation for any given symmetric positive semidefinite kernel.
\end{Remark}

\subsection{Kernel-based regression and classification}\label{sec: kernel-based learning}

It was shown in \cite{KLUS2019} that SINDy -- originally developed for learning the governing equations of dynamical systems~\cite{BRUNTON2016} -- can be applied to classification problems using the transformed data matrices (or, more generally, tensors) and resulting kernels introduced above. This holds for supervised learning problems in general as we will illustrate in this section. Given a training set $X \in \R^{d \times m}$ and the corresponding output matrix $Y \in \R^{d^\prime \times m}$, we consider the minimization problem
\begin{equation}\label{eq: minimization problem 1}
  \min_{\Theta \in \R^{d^\prime \times n}} \lVert Y - \Theta \Psi_X \rVert_F^2,
\end{equation}
where $\Psi_X \in \R^{n \times m}$ is the transformed data matrix given by the feature map $\Psi$, see Section~\ref{sec: tdt and kernels}. A solution $\Theta$ of \eqref{eq: minimization problem 1} then represents coefficients for expressing the output variables in terms of the preselected basis functions, i.e.,
\begin{equation*}
  y^{(j)}_i \approx \sum_{\mu=1}^n \Theta_{i,\mu} \psi_{\mu}(x^{(j)}).
\end{equation*}
If the number of basis functions is significantly larger than the number of training data points, typically equality holds. This might result in overfitting, i.e., the learned functions provide exact results for the training data but might not generalize well to new data. The regression/decision function $f \colon \mathcal{S} \to \R^{d^\prime}$ is defined as
\begin{equation}\label{eq: target function}
  f(x) = \Theta \Psi(x),
\end{equation}
where $x \in \mathcal{S} \subseteq \R^d$ is a sample from a given test set. For classification problems, the $i$th entry of the vector $f(x)$ is interpreted as the likelihood of $x$ belonging to category $i$. The above approach in combination with hard thresholding is known as \emph{sparse identification of nonlinear dynamical systems} (SINDy) in the context of system identification, i.e., when the data matrix $X$ contains snapshots of a dynamical system and $Y$ the corresponding derivatives. 

Adding a regularization term of the form $\gamma \lVert \Theta \rVert_F^2$ to~\eqref{eq: minimization problem 1} yields the so-called ridge regression problem~\cite{MURPHY2012}. The solution would then be given by $\Theta = Y (\Psi_X^\top \Psi_X + \gamma \mathrm{Id})^{-1} \Psi_X^\top$. As $\gamma$ goes to zero, this converges to the solution of \eqref{eq: minimization problem 1} with minimal Frobenius norm, i.e.,
\begin{equation}\label{eq: coefficient matrix 1}
  \Theta = Y \Psi_X^+,
\end{equation}
where $\Psi_X^+$ denotes the \emph{Moore--Penrose inverse} (or simply called \emph{pseudoinverse}) of $\Psi_X$, see~\cite{GOLUB2013}. The direct computation of $\Psi_X^+$ may be computationally expensive due to a large number of basis functions. Therefore, we use the kernel trick, i.e., the pseudoinverse is written as
\begin{equation*}
  \Psi_X^+ =  (\Psi_X ^\top \Psi_X)^+ \Psi_X^\top = G_{X,X}^+ \Psi_X^\top.
\end{equation*}
Now, the solution of \eqref{eq: minimization problem 1} can be expressed as
\begin{equation}\label{eq: coefficient matrix 2}
  \Theta = Y G_{X,X}^+ \Psi_X^\top. 
\end{equation}
The coefficient matrix $\Theta^\prime := Y G_{X,X}^+$ is then the minimum-norm solution of
\begin{equation}\label{eq: minimization problem 2}
  \min_{\Theta \in \R^{d^\prime \times m}} \lVert Y - \Theta G_{X,X} \rVert_F^2.
\end{equation}
This tells us that, equivalently to \eqref{eq: minimization problem 1}, we can consider the minimization problem \eqref{eq: minimization problem 2} in order to construct the target function
\begin{equation*}
  f(x) = \Theta \Psi(x) = \Theta^\prime \Psi_X^\top \Psi(x) = \Theta^\prime G_{X, x},
\end{equation*}
where $G_{X,x}$ denotes the Gram matrix corresponding to the sets $X$ and $\{x\}$.

In comparison to the coefficient matrix $\Theta$, the matrix $\Theta^\prime$ may have much smaller dimensions (assuming that $n \gg m$). Also, we only consider Gram matrices and kernel evaluations, there is no need to construct a transformed data matrix and we therefore may be able to mitigate storage problems. The above approach was introduced as kernel-based SINDy/MANDy in \cite{KLUS2019} and even works if the transformation $\Psi$ is not given explicitly, but implicitly defined in terms of a kernel function.

\section{Data reduction}\label{sec: Data reduction}

As described in the previous section, the kernel-based construction of the regression/decision function allows us to omit explicit computations in the feature space. Under this premise, we are therefore not interested in methods for finding approximate kernel expressions or basis-function reductions as done for example in~\cite{RAHIMI2008, LE2013, LITZINGER2018, SANTIN2019}. The crucial aspect, in our case, is the dimension of the involved Gram matrix, whose size is given by the number of considered samples. In this section, we introduce an approach to filter unnecessary samples in order to speed up the computations, in particular in the testing phase. 

\subsection{Feature space approximation}\label{sec: data-reduced regression}

We assume that the transformed data matrix is (nearly) rank-deficient or, in other words, there exists a $n \times \tilde{m}$-submatrix of $\Psi_X$ such that the column spaces of both matrices are (numerically) equal. Suppose there exist a subset $\tilde{X} \subset X$ and a matrix $C \in \R^{\tilde{m} \times m}$ with $\tilde{m} := |\tilde{X}| < m$ such that 
\begin{equation}\label{eq: linear comb psi}
  \Psi_{\tilde{X}} C =  \Psi_X,
\end{equation}
where the columns of $\Psi_{\tilde{X}} \in \R^{n \times \tilde{m}}$ are a linearly independent subset of the columns of $\Psi_X$. In practice, (almost) linear dependencies among the columns of the transformed data matrix $\Psi_X$ lead to ill-conditioned or singular Gram matrices. Thus, the proposed approach can also be seen as a regularization technique. 

\begin{Lemma}
  Given matrices $A \in \R^{k \times l}$ and $B \in \R^{l \times m}$ where $A$ has full column rank and $B$ full row rank, it holds that $(AB)^+ = B^+ A^+$.
\end{Lemma}
\noindent For a proof of the above lemma, we refer to~\cite{CAMPBELL2009}. By assuming that $C$ has linearly independent rows, the pseudoinverse of $\Psi_X$ is then given by
\begin{equation*}
 \Psi_X^+ = C^+ \Psi_{\tilde{X}}^+.
\end{equation*}
Using the kernel trick, the solution \eqref{eq: coefficient matrix 1} can now be written as
\begin{equation*}
  \Theta = Y C^+ \Psi_{\tilde{X}}^+ = Y C^+ G_{\tilde{X}, \tilde{X}}^+ \Psi_{\tilde{X}}^\top.
\end{equation*}
It holds that $\Theta = \tilde{\Theta} \Psi_{\tilde{X}}^\top$, where the matrix $\tilde{\Theta} := Y C^+ G_{\tilde{X}, \tilde{X}}^+$ is the minimum-norm solution of the least-squares problem
\begin{equation}\label{eq: minimization problem 3}
  \min_{\Theta \in \R^{d^\prime \times \tilde{m}}} \lVert Y - \Theta G_{\tilde{X} , \tilde{X}} C \rVert_F = \min_{\Theta \in \R^{d^\prime \times \tilde{m}}} \lVert Y - \Theta G_{\tilde{X} , X}\rVert_F
\end{equation}
because $G_{\tilde{X}, \tilde{X}}$ has full rank (equal to the column rank of $\Psi_{\tilde{X}}$) and therefore it holds that 
\begin{equation*}
  (G_{\tilde{X}, \tilde{X}} C)^+ = C^+ G_{\tilde{X}, \tilde{X}}^+.
\end{equation*} 
Since $\tilde{m} < m \ll n$, this means we can -- under the above assumptions -- solve a minimization problem with potentially much lower dimensions than \eqref{eq: coefficient matrix 2} and especially \eqref{eq: coefficient matrix 1}, but still can fully reconstruct the target function 
\begin{equation}\label{eq: minimization problem 3 - f}
  f(x) = \tilde{\Theta} G_{\tilde{X}, x} = \Theta \Psi(x)
\end{equation}
as defined in \eqref{eq: target function} (if the feature map is explicitly given). Note that even though the coefficient matrix $\tilde{\Theta}$ lives in a lower-dimensional space than $\Theta$, the minimization problem~\eqref{eq: minimization problem 3} still includes all training data for supervised learning. Especially in the test phase this can be advantageous. If $\tilde{m}$ is significantly smaller than $m$ and $n$, the evaluation of the regression/classification function is much cheaper since it only requires the multiplication of a $d^\prime \times \tilde{m}$ by a vector of $\tilde{m}$ kernel evaluations.

In general, of course, there does not have to exist such an optimal set $\tilde{X}$ satisfying~\eqref{eq: linear comb psi}. However, for practical applications we may find subsets $\tilde{X} \subset X$ which satisfy~\eqref{eq: linear comb psi} up to a certain extent, i.e., any feature vector $\Psi(x)$, $x \in X$, can be closely approximated as a linear combination of the columns of $\Psi_{\tilde{X}}$.

\subsection{Analytical formulation \& existing approaches}\label{sec: analytical formulation}

As described above, we seek a subset $\tilde{X}$ of the columns of the data matrix $X$ such that \eqref{eq: linear comb psi} is (approximately) satisfied. That is, for each vector in $x \in X \setminus \tilde{X}$ there should exist a vector $c \in \R^{\tilde{m}}$ such that $\Psi_{\tilde{X}} c$ is equal to (or very close to) $\Psi(x)$. In other words, we try to find a smallest possible $\tilde{X} \subset X$ such that all column vectors of $\Psi_X$ lie in the span of the column vectors of $\Psi_{\tilde{X}}$. For a single vector $\Psi(x)$ with $x \in X $, we define the approximation error by
\begin{equation}\label{eq: approximation error}
  E(\tilde{X}, x) := \min_{c \in \R^{\tilde{m}}} \lVert \Psi(x) - \Psi_{\tilde{X}} c \rVert^2_2,
\end{equation}
where the minimum-norm solution can be written as $c = \Psi_{\tilde{X}}^+ \Psi(x)$, i.e., $E(\tilde{X}, x) = \Psi(x) - \Psi_{\tilde{X}} \Psi^+_{\tilde{X}} \Psi(x)$. If $\Psi(x) \in \textrm{span} \, \Psi_{\tilde{X}}$, especially if $x \in \tilde{X}$, the approximation error~\eqref{eq: approximation error} is equal to zero. Combining the error minimization with the constraint that $\tilde{m} = |\tilde{X}|$ should be small in comparison to $m = |X|$ (assuming that the number of given observations is large enough), we have to solve the following optimization problem:
\begin{equation}\label{eq: analytical formulation}
\begin{split}
  \min_{\tilde{X} \subset X} \left( \sum_{x \in X }  E(\tilde{X}, x)  + \gamma \tilde{m} \right)
  &= \min_{\tilde{X} \subset X} \left(  \sum_{x \in X }  \min_{c \in \R^{\tilde{m}}} \lVert \Psi(x) - \Psi_{\tilde{X}} c \rVert^2_2  + \gamma \tilde{m} \right)\\ 
  &= \min_{\tilde{X} \subset X}  \left( \min_{C \in \R^{\tilde{m} \times m}} \lVert \Psi_{X} - \Psi_{\tilde{X}} C \rVert^2_F  + \gamma \tilde{m} \right),
\end{split}
\end{equation}
where $\gamma$ is a regularization parameter such that high values for $\gamma$ will enforce selecting fewer points. 

Finding an optimal solution of~\eqref{eq: analytical formulation} is a complex combinatorial optimization problem and there exist different iterative and greedy approaches. A description of \emph{reduced set methods}, where samples are removed from the set $X$ using, e.g., kernel PCA or $\ell_1$ penalization, can be found in~\cite{SCHOELKOPF2002}. However, these approaches are based on a stepwise reduction of the whole set of feature vectors. That is, in each iteration step the sample $x \in \tilde{X}$ with minimum error $E(\tilde{X} \setminus x, x)$ is removed from $\tilde{X}$ (starting with $\tilde{X} = X$). In particular, if the set of feature vectors is too large to handle at once, this method cannot be applied. Depending on the number of samples and the dimension of the feature space, reduced set methods can be computationally expensive in general. Another efficient way to find important samples in the given data set is the \emph{Nyström method}~\cite{WANG2013}. It can be generally applied to symmetric positive semidefinite matrices $A$ in order to compute a decomposition of the form $A = B C^+ B^\top$, where $B$ is a matrix formed by a subset of the columns of $A$ and $C$ is an intersection matrix. In particular, the Nyström method can be used to approximate Gram matrices~\cite{DRINEAS2005, WILLIAMS2000}. That is, for a fixed number of samples $\tilde{m}$, the Nyström algorithm seeks to find a subset $\tilde{X} \subset X$ with $|\tilde{X}| = \tilde{m}$ such that the Gram matrix $G_{X,X} = \Psi_X^\top \Psi_X$ can be approximately written as 
\begin{equation}\label{eq: Nystrom}
  G_{X,X} \approx G_{X,\tilde{X}} C,
\end{equation}
with $C = G_{\tilde{X}, \tilde{X}} ^+ G_{\tilde{X}, X}$. On the one hand, however, we here have a restriction on $C$ so that there generally may exist better choices for $\tilde{X}$ and $C$ to minimize the error $\lVert G_{X,X} - G_{X, \tilde{X}} C \rVert_F$. On the other hand, the Nyström approximation does not provide a feature space approximation as described in Section~\ref{sec: data-reduced regression} since~\eqref{eq: Nystrom} does not imply $\Psi_X \approx \Psi_{\tilde{X}} C$. Note that only the other direction is true.

In what follows, we will therefore propose a heuristic approach to iteratively construct a subset $\tilde{X}$ such that any feature vector corresponding to a sample $x \in X$ can be approximated by a linear combination of the set $\{\Psi(x) \colon x \in \tilde{X}\}$ up to an arbitrarily small error $\varepsilon$. In Section~\ref{sec: CalCOFI}, results obtained by applying the (recursive) Nyström method as introduced in~\cite{MUSCO2017} and our approach will be compared.

\begin{Remark}
 The Nyström method as well as, e.g., the \emph{iterative spectral method} (ISM)~\cite{WU2019} can also be used for the reduction of the feature space dimension. However, our aim is to reduce the number of samples/feature vectors $m$, not the dimension $n$.
\end{Remark}

\subsection{Heuristic approach}

In order to find a subset as described in Section~\ref{sec: analytical formulation}, we introduce a heuristic approach based on a \emph{bottom-up} principle -- in contrast to the reduced set methods described in~\cite{SCHOELKOPF2002}, which are based on a \emph{top-down} selection of the samples. Suppose we are given a subset $\tilde{X}$ of the columns of $X$ and the corresponding transformed data matrix $\Psi_{\tilde{X}}$. For the following approach -- which we call \emph{kernel-based feature space approximation} or in short \emph{kFSA} -- we expand the subset $\tilde{X}$ step by step by adding the vector $x$ that has the largest approximation error~\eqref{eq: approximation error}. The idea is to increase the dimension of the subspace spanned by the columns of $\Psi_{\tilde{X}}$ iteratively without including linearly dependent feature vectors (or even duplicates).

As the first data point, we choose the sample whose feature vector approximates all other columns of $\Psi_X$ in an optimal way, i.e., we consider the extreme case of the minimization problem~\eqref{eq: analytical formulation} where the regularization parameter $\gamma$ is sufficiently large so that only a single sample is selected. We then obtain
\begin{equation*}
x_0 = \argmin_{x \in X} \min_{c \in \R^{1 \times m}} \lVert \Psi_X - \Psi(x)c \rVert_F^2.  
\end{equation*}
It holds that
\begin{equation*}
  c = \Psi(x)^+ \Psi_X = \frac{1}{\lVert \Psi(x) \rVert_2 ^2} \Psi(x)^\top \Psi_X = \frac{1}{k(x,x)} G_{x, X}
\end{equation*}
and, therefore, $x_0$ is given by
\begin{equation}\label{eq: x_0}
\begin{split}
  x_0 &= \argmin_{x \in X} \lVert \Psi_X - \frac{1}{k(x,x)} \Psi(x) G_{x,X} \rVert_F^2\\
  &= \argmin_{x \in X} \textrm{tr} \left( \left(  \Psi_X - \frac{1}{k(x,x)} \Psi(x) G_{x,X}\right)^\top \left(  \Psi_X - \frac{1}{k(x,x)} \Psi(x) G_{x,X}\right) \right)\\
  &= \argmin_{x \in X} \textrm{tr} \left(  G_{X,X} - \frac{2}{k(x,x)} G_{X,x}G_{x,X} + \frac{1}{k(x,x)^2} G_{X,x}  k(x,x) G_{x,X}\right)\\
  &= \argmin_{x \in X} \textrm{tr} \left(  G_{X,X} - \frac{1}{k(x,x)} G_{X,x}G_{x,X} \right)\\
  &= \argmin_{x \in X}  \sum_{x^\prime \in X} k(x^\prime, x^\prime) - \frac{k(x,x^\prime)^2}{k(x,x)}\\
  &= \argmax_{x \in X} \sum_{x^\prime \in X} \frac{k(x,x^\prime)^2}{k(x,x)}.
\end{split}
\end{equation}
As long as $\Psi_{\tilde{X}} \in \R^{n \times \tilde{m}}$, $n \geq \tilde{m}$, has full column rank, it holds that $G_{\tilde{X}, \tilde{X}}^+ = G_{\tilde{X}, \tilde{X}}^{-1}$ since $G_{\tilde{X}, \tilde{X}}$ is then a non-singular matrix. Thus, we can rewrite the approximation error~\eqref{eq: approximation error} in terms of kernel functions and corresponding Gram matrices.

\begin{Lemma}\label{lemma: kernel-based approximation error}
  Using the fact that $\Psi_{\tilde{X}}^+ = G_{\tilde{X}, \tilde{X}}^{-1} \Psi_{\tilde{X}}^\top$, the approximation error~\eqref{eq: approximation error} for any $x \in X$ can be written as 
  \begin{equation}\label{eq: kernel-based approximation error}
    E(\tilde{X}, x) = k(x,x) -  G_{x, \tilde{X}} G_{\tilde{X}, \tilde{X}}^{-1} G_{\tilde{X}, x}.
  \end{equation}

\end{Lemma}
\begin{proof} Using $c = \Psi_{\tilde{X}}^+ \Psi(x)$, we obtain
 \begin{equation*}
  \begin{split}
    E(\tilde{X}, x) &= \lVert \Psi(x) - \Psi_{\tilde{X}} \Psi_{\tilde{X}}^+ \Psi(x) \rVert ^2_2 \\
    & = \lVert \Psi(x) - \Psi_{\tilde{X}} G_{\tilde{X}, \tilde{X}}^{-1} \Psi_{\tilde{X}}^\top \Psi(x) \rVert ^2_2\\
    &= \left(\Psi(x) - \Psi_{\tilde{X}} G_{\tilde{X}, \tilde{X}}^{-1} \Psi_{\tilde{X}}^\top \Psi(x) \right)^\top \left(\Psi(x) - \Psi_{\tilde{X}} G_{\tilde{X}, \tilde{X}}^{-1} \Psi_{\tilde{X}}^\top \Psi(x) \right) \\
    &= \Psi(x)^\top \Psi(x) - 2 \Psi(x)^\top \Psi_{\tilde{X}} G_{\tilde{X}, \tilde{X}}^{-1} \Psi_{\tilde{X}}^\top \Psi(x) + \Psi(x)^\top \Psi_{\tilde{X}} G_{\tilde{X}, \tilde{X}}^{-1} \underbrace{\Psi_{\tilde{X}}^\top \Psi_{\tilde{X}}}_{=G_{\tilde{X}, \tilde{X}}} G_{\tilde{X}, \tilde{X}}^{-1} \Psi_{\tilde{X}}^\top \Psi(x) \\
    &= k(x,x) -  G_{x, \tilde{X}} G_{\tilde{X}, \tilde{X}}^{-1} G_{\tilde{X}, x}. \qedhere
  \end{split}
\end{equation*}
\end{proof}
\noindent Let us denote the vector containing all approximation errors for the samples in $X \setminus \tilde{X}$ by $\Delta$. Using the Hadamard product, this vector can be expressed in a compact way as the following lemma shows.

\begin{Lemma}\label{lemma: error vector - sle}
 The vector $\Delta \in \R^{m - \tilde{m}}$ can be written as
\begin{equation}\label{eq: error vector}
  \Delta = \emph{\textrm{diag}}\left(G_{X\setminus \tilde{X}, X\setminus \tilde{X}}\right)^\top - \mathds{1}^\top \cdot \left(G_{\tilde{X}, X \setminus \tilde{X} } \odot Z\right),
\end{equation}
where $\mathds{1} \in \R^{\tilde{m}}$ is a vector of ones and $Z$ is the unique solution of 
\begin{equation}\label{eq: error vector - Z}
  G_{\tilde{X}, \tilde{X}} Z = G_{\tilde{X}, X \setminus \tilde{X}}.
\end{equation}
\end{Lemma}
\begin{proof}
  For any $x \in X$, the approximation error~\eqref{eq: kernel-based approximation error} can be expressed as
  \begin{equation*}
  \begin{split}
    E(\tilde{X}, x) &= k(x,x) - \sum_{i=1}^{\tilde{m}} \left( G_{\tilde{X}, x} \right)_i \left( G_{\tilde{X}, \tilde{X}}^{-1} G_{\tilde{X},x} \right)_i \\
    &= k(x,x) - \sum_{i=1}^{\tilde{m}} \left( G_{\tilde{X}, x} \odot  \left( G_{\tilde{X}, \tilde{X}}^{-1} G_{\tilde{X},x} \right) \right) _i\\
    &= k(x,x) - \mathds{1}^\top \cdot  \left( G_{\tilde{X}, x} \odot  \left( G_{\tilde{X}, \tilde{X}}^{-1} G_{\tilde{X},x} \right) \right).
  \end{split}
  \end{equation*}
  For the error vector $\Delta$, it follows that
  \begin{equation*}
    \Delta = \emph{\textrm{diag}}\left(G_{X\setminus \tilde{X}, X\setminus \tilde{X}}\right)^\top - \mathds{1}^\top \cdot \Big(G_{\tilde{X}, X \setminus \tilde{X} } \odot \underbrace{G_{\tilde{X}, \tilde{X}}^{-1} G_{\tilde{X}, X \setminus \tilde{X}}}_{=: Z}\Big). \qedhere
  \end{equation*}
\end{proof}
\noindent Note that we have to construct the complete Gram matrix $G_{X,X}$ corresponding to all samples only once. For the computation of the approximation errors in each step, we could then simply extract the necessary submatrices. However, using Lemma~\ref{lemma: error vector - sle} for the error computation would result in a high complexity since the dominant computational costs arise from solving the systems of linear equations given in~\eqref{eq: error vector - Z}. That is, supposing we have already found $\tilde{m}$ samples, we then have to solve a system with an $\tilde{m} \times \tilde{m}$-matrix and $m-\tilde{m}$ right-hand sides. Thus, the overall complexity can be estimated by $O \left( \sum_{\tilde{m}=1}^{M} \tilde{m}^3 (m-\tilde{m}) \right)$ if we repeat the procedure $M$ times until the desired accuracy is reached.

In order to reduce the computational effort, we can use an alternative way of calculating the approximation errors. Instead of computing the matrix $Z$ that is needed for the construction of $\Delta$ in~\eqref{eq: error vector} as the solution of the system of linear equations given in~\eqref{eq: error vector - Z}, it is possible to iteratively calculate $G^{-1}_{\tilde{X}, \tilde{X}}$ in order to compute $Z$ directly by matrix multiplication. 

\begin{Lemma}\label{lemma: blockwise inversion}
  Given matrices $A,B,C,D$ of appropriate size such that $A$ and $D-CA^{-1}B$ are non-singular square matrices, the blockwise inversion formula (also called \emph{Banachiewicz inversion formula}) states
\begin{equation*}
{\begin{bmatrix}A&B\\C&D\end{bmatrix}}^{-1}={\begin{bmatrix}A^{-1}+A^{-1}BS^{-1}CA^{-1}&-A^{-1}BS^{-1}\\-S^{-1}CA^{-1}&S^{-1}\end{bmatrix}},
\end{equation*}
where the matrix $S$ -- called the \emph{Schur complement} -- is given by $S = D-CA^{-1} B$.
\end{Lemma}
\noindent For a proof of the above statement, we refer to~\cite{TIAN2009}. Suppose we want to expand the (linearly independent) set $\tilde{X}$ by a sample $x_\text{new}$ and the inverse of $G_{\tilde{X}, \tilde{X}}$ is known -- which is of course the case for the initial sample. The Gram matrix corresponding to the set $\tilde{X} \cup \{x_\text{new}\}$ is then given by
\begin{equation}\label{eq: block matrix}
  G_{\tilde{X} \cup \{x_\text{new}\}, \tilde{X} \cup \{x_\text{new}\}} = \begin{bmatrix}\Psi_{\tilde{X}}^\top \\ \Psi(x_\text{new})^\top \end{bmatrix} \begin{bmatrix}\Psi_{\tilde{X}} & \Psi(x_\text{new})\end{bmatrix} = 
  \begin{bmatrix} G_{\tilde{X}, \tilde{X}} & G_{\tilde{X}, x_\text{new}} \\ G_{x_\text{new}, \tilde{X}} & G_{x_\text{new},x_\text{new}} \end{bmatrix}, 
\end{equation}
where $G_{x_\text{new},x_\text{new}}$ is simply the kernel evaluation at $x_\text{new}$, i.e., $G_{x_\text{new},x_\text{new}} = k(x_\text{new},x_\text{new})$. If the feature vector $\Psi(x_\text{new})$ is not an element of the span of the columns of $\Psi_{\tilde{X}}$, then the matrix~\eqref{eq: block matrix} is non-singular and, furthermore, symmetric positive definite. Applying Lemma~\ref{lemma: blockwise inversion}, we can write
\begin{equation}\label{eq: block inverse}
  G_{\tilde{X} \cup \{x_\text{new}\}, \tilde{X} \cup \{x_\text{new}\}}^{-1} = 
  \begin{bmatrix}
    G_{\tilde{X}, \tilde{X}}^{-1} + \frac{T T^\top}{S} & -\frac{T}{S} \\
    -\frac{T^\top}{S} & \frac{1}{S}
  \end{bmatrix}
\end{equation}
with $T = G_{\tilde{X}, \tilde{X}}^{-1} G_{\tilde{X}, x_\text{new}}$ and $S = k(x_\text{new}, x_\text{new}) - G_{x_\text{new}, \tilde{X}} G_{\tilde{X}, \tilde{X}}^{-1} G_{\tilde{X}, x_\text{new}} = E(\tilde{X},x_\text{new}) > 0 $. Using~\eqref{eq: block inverse} would already avoid having to solve a system of linear equations. In this case, the dominant computational effort, namely the multiplication of $G_{\tilde{X}, \tilde{X}}^{-1}$ and $G_{\tilde{X}, X \setminus \tilde{X}}$ in every iteration step, can be estimated by $O\left( \sum_{\tilde{m}=1}^{M} \tilde{m}^2 (m - \tilde{m})\right)$. However, we can exploit the block inverse~\eqref{eq: block inverse} to directly construct an iterative scheme for $Z$ and therefore reduce the complexity of our method even further. Each column of the matrix $Z= G^{-1}_{\tilde{X}, \tilde{X}} G_{\tilde{X}, X \setminus \tilde{X}} \in \R^{\tilde{m} \times (m - \tilde{m})}$ is from now on considered as the evaluation of a function $H$ of a set $\tilde{X}$ and a sample $x \in X \setminus \tilde{X}$, i.e., $H(\tilde{X}, x) = G^{-1}_{\tilde{X}, \tilde{X}} G_{\tilde{X},x}$. 
\begin{theorem}
 For any $x_\text{new}$ with $\Psi(x_\text{new}) \notin \emph{\textrm{span}}\,\Psi_{\tilde{X}}$ and $x \in X$, it holds that
 \begin{equation*}
   H(\tilde{X} \cup \{x_\text{new}\}, x) = \begin{bmatrix} H(\tilde{X}, x) -  H(\tilde{X}, x_\text{new}) \lambda\\ \lambda \end{bmatrix} 
 \end{equation*}
 with $\lambda = \frac{k( x_\text{new}, x) - G_{x_\text{new}, \tilde{X}} H(\tilde{X}, x)}{E(\tilde{X}, x_\text{new})}$.
\end{theorem}
\begin{proof}
 By definition, we have
 \begin{equation*}
   H(\tilde{X} \cup \{x_\text{new}\}, x) = G_{\tilde{X} \cup \{x_\text{new}\}, \tilde{X} \cup \{x_\text{new}\}}^{-1} G_{\tilde{X} \cup \{x_\text{new}\}, x}.
 \end{equation*}
 Using~\eqref{eq: block inverse}, it follows that
 \begin{equation*}
  H(\tilde{X} \cup \{x_\text{new}\}, x) = \begin{bmatrix}
    G_{\tilde{X}, \tilde{X}}^{-1} + \frac{T T^\top}{S} & -\frac{T}{S} \\
    -\frac{T^\top}{S} & \frac{1}{S}
  \end{bmatrix} \begin{bmatrix} G_{\tilde{X} , x} \\ k( x_\text{new}, x) \end{bmatrix},
 \end{equation*}
 with $T = G_{\tilde{X}, \tilde{X}}^{-1} G_{\tilde{X}, x_\text{new}} = H(\tilde{X}, x_\text{new})$ and $S = E(\tilde{X}, x_\text{new})$, see Lemma~\ref{lemma: kernel-based approximation error}. Therefore, we obtain
 \begin{equation*}
 \begin{split}
   H(\tilde{X} \cup \{x_\text{new}\}, x) &= \begin{bmatrix} G_{\tilde{X}, \tilde{X}}^{-1} G_{\tilde{X} , x} +  H(\tilde{X}, x_\text{new}) \left(   \frac{G_{x_\text{new}, \tilde{X}} G_{\tilde{X}, \tilde{X}}^{-1} G_{\tilde{X} , x} - k(x_\text{new}, x)}{E(\tilde{X}, x_\text{new})} \right) \\ \frac{k(x_\text{new}, x) - G_{x_\text{new}, \tilde{X}} G_{\tilde{X}, \tilde{X}}^{-1} G_{\tilde{X}, x}}{E(\tilde{X}, x_\text{new})} \end{bmatrix}\\
   &= \begin{bmatrix} H(\tilde{X} , x) - H(\tilde{X} , x_\text{new}) \left(   \frac{ k( x_\text{new}, x) -G_{x_\text{new}, \tilde{X}} H(\tilde{X} , x) }{E(\tilde{X}, x_\text{new})} \right) \\ \frac{ k(x_\text{new}, x) -G_{x_\text{new}, \tilde{X}} H(\tilde{X} , x) }{E(\tilde{X}, x_\text{new})} \end{bmatrix},
 \end{split}
 \end{equation*}
which yields the assertion. Note that $H(\tilde{X}, x) \in \R^{\tilde{m}}$ is \emph{not} a subvector of $H(\tilde{X} \cup x_\text{new}, x) \in \R^{\tilde{m}+1} $.
\end{proof}

\newcommand{\restrictedto}[2]{#1_{\mkern 1mu \vrule depth -1.5ex height 3ex \mkern 2mu \smash{\raisebox{1.5ex}{\scriptsize$#2$}}}}

\noindent Considering $Z = [H(\tilde{X}, x)]_{x \in X \setminus \tilde{X}} \in \R^{\tilde{m} \times (m - \tilde{m})} $ as given in \eqref{eq: error vector - Z}, we define
\begin{equation}\label{eq: Lambda}
 \Lambda = [\lambda]_{x \in X \setminus \tilde{X}} = \frac{1}{E(\tilde{X}, x_\text{new})}  \left( G_{x_\text{new}, X \setminus \tilde{X}} - G_{x_\text{new}, \tilde{X}} Z \right) \in \R^{m - \tilde{m}}.
\end{equation}
Then, the matrix 
\begin{equation*}
 Z_\text{new} = G_{\tilde{X} \cup \{x_\text{new}\}, \tilde{X} \cup \{x_\text{new}\}}^{-1} G_{\tilde{X} \cup \{x_\text{new}\}, X \setminus (\tilde{X} \cup \{x_\text{new}\})} \in \R^{\tilde{m}+1 \times (m - \tilde{m} -1 )}
\end{equation*}
required for the error computation~\eqref{eq: error vector} when considering the updated set $\tilde{X} \cup \{x_\text{new}\}$ instead of $\tilde{X}$ can be written as
\begin{equation}\label{eq: Z_new}
 Z_\text{new} = \begin{bmatrix} Z -  Z_{| \raisebox{-1.5pt}{$ \scriptstyle x_\text{new} $}} \Lambda^\top \\[0.2cm] \Lambda^\top \end{bmatrix}_{\big| \raisebox{-3pt}{$ \scriptstyle X \setminus (\tilde{X} \cup \{x_\text{new}\})$}},
\end{equation}
where $Z_{| \raisebox{-1.5pt}{$ \scriptstyle x_\text{new} $}} = H(\tilde{X}, x_\text{new})$ denotes the column of $Z$ corresponding to the sample $x_\text{new} \in X \setminus \tilde{X}$. Note that the matrix computed in~\eqref{eq: Z_new} is also restricted to the columns corresponding to the set ${X \setminus (\tilde{X} \cup \{x_\text{new}\})}$. We combine the iterative construction of the matrix $Z$ with the error-vector computation given in~\eqref{eq: error vector}. The above derivations are summarized in Algorithm~\ref{alg: sample reduction}.

\begin{Lemma}\label{lemma: kFSA complexity}
  Neglecting the computational costs for the construction of $G_{X,X}$, the complexity of Algorithm~\ref{alg: sample reduction} can be estimated by $O( m^2 + m M^2 )$, where $m$ is the number of given training samples and $M$ the number of extracted samples.
\end{Lemma}
\begin{proof}
  Determining the initial sample in line~\ref*{alg: sample reduction - initial} needs $O(m^2)$ operations. For every execution of line~\ref*{alg: sample reduction - errors}, the dominant computational costs arise from calculating the Hadamard product in~\eqref{eq: error vector}, i.e., ${O(\tilde{m} (m-\tilde{m}))}$. Since the error $E(\tilde{X}, x_\text{new})$ is already known in form of $\delta_\text{new}$, the complexity of line~\ref*{alg: sample reduction - update Z} is divided into the computations of $\Lambda$ and $Z_\text{new}$ as given in~\eqref{eq: Lambda} and \eqref{eq: Z_new}, respectively. Both can be estimated by $O(\tilde{m} (m-\tilde{m}))$ as well. Repeating lines~\ref*{alg: sample reduction - errors} to \ref*{alg: sample reduction - update Z} of the above algorithm $M$ times, i.e., until the desired accuracy is reached, and adding all computational costs yields \begin{equation*}
  \begin{split}
   O\left(m^2 + \sum_{\tilde{m}=1}^M \tilde{m} (m -\tilde{m})\right) &= O \left( m^2 + m \sum_{\tilde{m}=1}^M  \tilde{m} - \sum_{\tilde{m}=1}^M \tilde{m}^2 \right) \\
   &= O \left( m^2 + \frac{m M (M+1)}{2} - \frac{M (M+1) (2M+1)}{6} \right) \\
   &= O \left( m^2 + M (M+1) \left( \frac{m}{2} - \frac{2M+1}{6} \right) \right),
   \end{split}
  \end{equation*}
  which can be roughly summarized as $O(m^2+ m M^2)$. Note that this estimation does not include a potential speed up due to removing samples for which the approximation error is already low enough, see line~\ref*{alg: sample reduction - remove}.
\end{proof}

\begin{mdframed}[backgroundcolor=boxback,hidealllines=true]
\vspace*{-2ex}
\begin{algorithm}[H]
  \caption{Kernel-based feature space approximation (kFSA).}
  \label{alg: sample reduction}
  \begin{spacing}{1.2}
  \begin{tabular}{ll}
    \textbf{Input:} & data matrix $X$, kernel function $k$, and threshold $\varepsilon>0$.\\
    \textbf{Output:} & subset $\tilde{X} \subset X$ with $E(\tilde{X}, x)<\varepsilon$ for all $x \in X$.
  \end{tabular}
  \vspace{0.2cm}\hrule\vspace{0.2cm}
  \begin{algorithmic}[1]
      \State Set $\tilde{X} := \{x_0\}$, see \eqref{eq: x_0}, and $X_\text{left} := X \setminus \{x_0\}$.
      \label{alg: sample reduction - initial}
      \State Define $Z := G_{\tilde{X}, \tilde{X}}^{-1} G_{\tilde{X}, X_\text{left}} = (1/k(x_0, x_0)) G_{x_0, X_\text{left}}$.
      \While{$X_\text{left}$ is not empty}
      \State Use $Z$ to compute the vector $\Delta \in \R^{|X_\text{left}|}$ of approximation errors, see \eqref{eq: error vector}.\label{alg: sample reduction - errors}
      \State Define $x_\text{new} \in X$ to be the sample corresponding to the maximum error $\delta_\text{new}$ in $\Delta$.
      \State Remove all $x$ from $X_\text{left}$ with corresponding approximation errors smaller than $\varepsilon$.\label{alg: sample reduction - remove}
      \If{$\delta_\text{new} \geq \varepsilon$}
      \State Add sample vector $x_\text{new}$ to $\tilde{X}$ and remove it from $X_\text{left}$.
      \label{alg: sample reduction - add sample}
      \State Update $Z$ as described above in~\eqref{eq: Z_new}.\label{alg: sample reduction - update Z}
      \EndIf
      \EndWhile
  \end{algorithmic}
  \end{spacing}
\end{algorithm}
\vspace*{-2ex}%
\end{mdframed}
\vspace{0.5cm}

When Algorithm~\ref{alg: sample reduction} terminates, all given samples $x \in X$ which are not an element of $\tilde{X}$ can be approximated with an error $E(\tilde{X}, x)$ smaller than $\varepsilon$. In case the Gram matrix corresponding to the whole training set consumes too much memory, the algorithm may be applied iteratively to subsets of the training set in order to filter unnecessary samples. Moreover, suppose we have found a subset $\tilde{X} \subset X$ by applying Algorithm~\ref{alg: sample reduction}. If we subsequently want to extend the training data set, one of the main advantages of our method in comparison to, e.g., reduced set methods is that it is easy to decide whether or not a new sample is added to $\tilde{X}$. We do not have to apply the algorithm again to the extended set but can simply continue our bottom-up approach.

\section{Results}\label{sec: Results}

In this section, we will present three examples for the application of data-reduced regression and classification, namely the MNIST data set, the Fermi--Pasta--Ulam--Tsingou problem, and hydrographic data from the California Cooperative Oceanic Fisheries Investigations. The numerical experiments have been performed on a Linux machine with 128 GB RAM and a 3 GHz Intel Xeon processor with 8 cores. The algorithms have been implemented in Matlab R2018b and are available on GitHub: \url{https://github.com/PGelss/kFSA}.

\subsection{MNIST data set}

Let us consider the classification of the MNIST data set~\cite{LECUN1998} as a first example. The set contains grayscale images of handwritten digits and their corresponding labels from $0$ to $9$ (represented using one-hot encoding). Originally, the images have a resolution of $28 \times 28$ pixels and are divided into $60000$ training and $10000$ test samples.\!\footnote{\url{http://yann.lecun.com/exdb/mnist/}} We downsample the images to $14 \times 14$ pixels by averaging groups of four pixels, cf.~\cite{STOUDENMIRE2016}. Additionally, we normalize the samples such that the largest pixel value in each image is equal to $1$. This has shown to yield better classification rates in our experiments. In~\cite{KLUS2019}, we already used the kernel-based methods described in Section~\ref{sec: kernel-based learning} to classify the downsampled MNIST images. Due to the large amount of training data, this example set is an ideal candidate for applying our proposed set-reduction method.

Given a subset of mutually different pixels $i_1, \dots, i_r \in \{1, \dots , 14^2\}$, we again use feature maps $\PSI \colon \R^r \to \R^{2 \times 2 \times \dots \times 2}$ of the form
\begin{equation}\label{eq: MNIST feature map}
 \PSI(x) = \Psi(x_{i_1}) \otimes \Psi(x_{i_2}) \otimes \dots \otimes \Psi(x_{i_r})= \begin{bmatrix}\cos(\kappa x_{i_1}) \\ \sin(\kappa x_{i_1})\end{bmatrix} \otimes \begin{bmatrix}\cos(\kappa x_{i_2}) \\ \sin(\kappa x_{i_2})\end{bmatrix} \otimes \dots \otimes \begin{bmatrix}\cos(\kappa x_{i_r}) \\ \sin(\kappa x_{i_r})\end{bmatrix},
\end{equation}
where $x_{i_1}, \dots, x_{i_r}$ are the corresponding pixel values and $\kappa>0$ is the kernel parameter. In contrast to, e.g., \cite{KLUS2019, STOUDENMIRE2016}, we do not consider all pixels of a given image at once. Instead, we partition the images into $9$ blocks, see Figure~\ref{fig: MNIST partition}, and then apply feature maps as given in~\eqref{eq: MNIST feature map} separately. The feature map~\eqref{eq: MNIST feature map}  is defined in terms of tensor products. That is, the entries of $\PSI(x)$ are given by products of $r$ cosine or sine functions evaluated at the given pixels, i.e.,
\begin{equation*}
  \PSI(x)_{b_1, \dots, b_r} = \left( \Psi(x_{i_1}) \right)_{b_1} \cdot \left( \Psi(x_{i_2}) \right)_{b_2} \cdot \ldots \cdot \left( \Psi(x_{i_r}) \right)_{b_r},
\end{equation*}
where 
\begin{equation*}
 \left( \Psi(x_{i_j}) \right)_{b_j} = \begin{cases} \cos(\kappa x_{i_j}), & \text{if }b_j = 1, \\ \sin(\kappa x_{i_j}), & \text{if }b_j =2.\end{cases}
\end{equation*}
The corresponding kernel is then given by
\begin{equation*}
  \PSI(x)^\top \PSI(x^\prime) = \bigotimes_{j=1}^r \left( \begin{bmatrix}\cos(\kappa x_{i_j}) & \sin(\kappa x_{i_j})\end{bmatrix} \cdot \begin{bmatrix}\cos(\kappa x^\prime_{i_j}) \\ \sin(\kappa x^\prime_{i_j})\end{bmatrix} \right) = \prod_{j=1}^r \cos(\kappa (x_{i_j} - x^\prime_{i_j}) ).
\end{equation*}

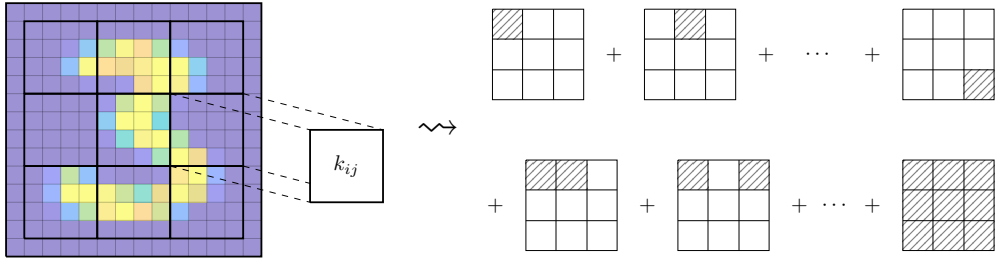
\begin{figure}
  \centering
  \scalebox{0.8}{
    \begin{tikzpicture}
    \def\h{4.2}
    \node[inner sep = 0, outer sep = 0, anchor=south west, opacity=0.5] at (0,0) {\includegraphics[width=4.2cm, height=4.2cm]{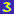}};
    \foreach \x in {1,2,3,4,5,6,7,8,9,10,11,12,13} {\draw[opacity=0.25] (0.3*\x, 0) --++ (0,\h);}  
    \foreach \x in {1,2,3,4,5,6,7,8,9,10,11,12,13} {\draw[opacity=0.25] (0, 0.3*\x) --++ (\h,0);}  
    \draw[line width=0.03cm] (0,0) --++ (0,\h) --++ (\h,0) --++ (0,-\h) --++ (-\h,0);
    \def\h{1.2}
    \draw[line width=0.03cm] (0.3,0.3) --++ (0,\h) --++ (\h,0) --++ (0,-\h) --++ (-\h,0);
    \draw[line width=0.03cm] (0.3,1.5) --++ (0,\h) --++ (\h,0) --++ (0,-\h) --++ (-\h,0);
    \draw[line width=0.03cm] (0.3,2.7) --++ (0,\h) --++ (\h,0) --++ (0,-\h) --++ (-\h,0);
    \draw[line width=0.03cm] (1.5,0.3) --++ (0,\h) --++ (\h,0) --++ (0,-\h) --++ (-\h,0);
    \draw[line width=0.03cm] (1.5,1.5) --++ (0,\h) --++ (\h,0) --++ (0,-\h) --++ (-\h,0);
    \draw[line width=0.03cm] (1.5,2.7) --++ (0,\h) --++ (\h,0) --++ (0,-\h) --++ (-\h,0);
    \draw[line width=0.03cm] (2.7,0.3) --++ (0,\h) --++ (\h,0) --++ (0,-\h) --++ (-\h,0);
    \draw[line width=0.03cm] (2.7,1.5) --++ (0,\h) --++ (\h,0) --++ (0,-\h) --++ (-\h,0);
    \draw[line width=0.03cm] (2.7,2.7) --++ (0,\h) --++ (\h,0) --++ (0,-\h) --++ (-\h,0);
    
    \draw[dashed] (2.7,1.5) -- (5,0.9);
    \draw[dashed] (3.9,1.5) -- (6.2,0.9);
    \draw[dashed] (2.7,2.7) -- (5,2.1);
    \draw[dashed] (3.9,2.7) -- (6.2,2.1);
    \draw[fill=white, line width=0.03cm] (5,0.9) --++ (0,\h) --++ (\h,0) --++ (0,-\h) --++ (-\h,0);
    \node[anchor=center] at (5.6,1.5) {$k_{ij}$};
    \node[] at (7.1,2.1) {\scalebox{2}{$\rightsquigarrow$}};
    
    \def\x{8}
    \def\y{4.1}
    \def\h{1.5}
    \draw[pattern=north east lines, pattern color=black!50, draw=none, inner sep=0, outer sep=0] (\x,\y) rectangle (\x+0.33*\h,\y-0.33*\h);
    \draw[] (\x,\y) --++ (\h,0) --++ (0,-\h) --++ (-\h,0) --++ (0,\h);
    \draw[] (\x+0.33*\h,\y) --++ (0,-\h);
    \draw[] (\x+0.67*\h,\y) --++ (0,-\h);
    \draw[] (\x,\y-0.33*\h) --++ (\h,0);
    \draw[] (\x,\y-0.67*\h) --++ (\h,0);

    \node[] at (10, 3.35) {$+$};
    
    \def\x{10.5}
    \def\y{4.1}
    \def\h{1.5}
    \draw[pattern=north east lines, pattern color=black!50, draw=none, inner sep=0, outer sep=0] (\x+0.33*\h,\y) rectangle (\x+0.67*\h,\y-0.33*\h);
    \draw[] (\x,\y) --++ (\h,0) --++ (0,-\h) --++ (-\h,0) --++ (0,\h);
    \draw[] (\x+0.33*\h,\y) --++ (0,-\h);
    \draw[] (\x+0.67*\h,\y) --++ (0,-\h);
    \draw[] (\x,\y-0.33*\h) --++ (\h,0);
    \draw[] (\x,\y-0.67*\h) --++ (\h,0);

    \node[] at (12.5, 3.35) {$+$};
    \node[] at (13.375, 3.35) {$\cdots$};
    \node[] at (14.25, 3.35) {$+$};
    
    \def\x{14.75}
    \def\y{4.1}
    \def\h{1.5}
    \draw[pattern=north east lines, pattern color=black!50, draw=none, inner sep=0, outer sep=0] (\x+0.67*\h,\y-0.67*\h) rectangle (\x+\h,\y-\h);
    \draw[] (\x,\y) --++ (\h,0) --++ (0,-\h) --++ (-\h,0) --++ (0,\h);
    \draw[] (\x+0.33*\h,\y) --++ (0,-\h);
    \draw[] (\x+0.67*\h,\y) --++ (0,-\h);
    \draw[] (\x,\y-0.33*\h) --++ (\h,0);
    \draw[] (\x,\y-0.67*\h) --++ (\h,0);

    \node[] at (8.05, 0.85) {$+$};
    
    \def\x{8.55}
    \def\y{1.6}
    \def\h{1.5}
    \draw[pattern=north east lines, pattern color=black!50, draw=none, inner sep=0, outer sep=0] (\x,\y) rectangle (\x+0.67*\h,\y-0.33*\h);
    \draw[] (\x,\y) --++ (\h,0) --++ (0,-\h) --++ (-\h,0) --++ (0,\h);
    \draw[] (\x+0.33*\h,\y) --++ (0,-\h);
    \draw[] (\x+0.67*\h,\y) --++ (0,-\h);
    \draw[] (\x,\y-0.33*\h) --++ (\h,0);
    \draw[] (\x,\y-0.67*\h) --++ (\h,0);
    
    \node[] at (10.55, 0.85) {$+$};
    
    \def\x{11.05}
    \def\y{1.6}
    \def\h{1.5}
    \draw[pattern=north east lines, pattern color=black!50, draw=none, inner sep=0, outer sep=0] (\x,\y) rectangle (\x+0.33*\h,\y-0.33*\h);
    \draw[pattern=north east lines, pattern color=black!50, draw=none, inner sep=0, outer sep=0] (\x+0.67*\h,\y) rectangle (\x+\h,\y-0.33*\h);
    \draw[] (\x,\y) --++ (\h,0) --++ (0,-\h) --++ (-\h,0) --++ (0,\h);
    \draw[] (\x+0.33*\h,\y) --++ (0,-\h);
    \draw[] (\x+0.67*\h,\y) --++ (0,-\h);
    \draw[] (\x,\y-0.33*\h) --++ (\h,0);
    \draw[] (\x,\y-0.67*\h) --++ (\h,0);
    
    \node[] at (13.05, 0.85) {$+$};
    \node[] at (13.65, 0.85) {$\cdots$};
    \node[] at (14.25, 0.85) {$+$};
    
    \def\x{14.75}
    \def\y{1.6}
    \def\h{1.5}
    \draw[pattern=north east lines, pattern color=black!50, draw=none, inner sep=0, outer sep=0] (\x,\y) rectangle (\x+\h,\y-\h);
    \draw[] (\x,\y) --++ (\h,0) --++ (0,-\h) --++ (-\h,0) --++ (0,\h);
    \draw[] (\x+0.33*\h,\y) --++ (0,-\h);
    \draw[] (\x+0.67*\h,\y) --++ (0,-\h);
    \draw[] (\x,\y-0.33*\h) --++ (\h,0);
    \draw[] (\x,\y-0.67*\h) --++ (\h,0);

    \end{tikzpicture}}
  \caption{Block decomposition of the MNIST data set: Each image is divided into $9$ disjoint $4 \times 4$ blocks where the pixels at the margin are neglected since their corresponding values are zero for almost all images in the data set. Then, a kernel function is defined for each subgroup of the pixels. The kernel combinations encoded in the Gram matrix $G$~\eqref{eq: MNIST Gram matrix} are depicted on the right-hand side.}
  \label{fig: MNIST partition}
\end{figure}

\noindent Thus, we consider $9$ kernel functions $k_1, \dots, k_9$ corresponding to different subgroups of the pixels as shown in Figure~\ref{fig: MNIST partition}, compute the associated Gram matrices $G_1, \dots, G_9 \in \R^{m \times m}$ and define
\begin{equation}\label{eq: MNIST Gram matrix}
\begin{split}
  G = \frac{1}{511} \left( (G_1 + \mathds{1}) \odot \ldots \odot(G_9 + \mathds{1}) - \mathds{1} \right)= \frac{1}{511} \sum_{p \in I} G_1 ^{\odot p_1} \odot G_2^{ \odot p_2} \odot \ldots \odot G_9^{\odot p_9},
\end{split}
\end{equation}
where $\mathds{1} \in \R^{m \times m}$ denotes the matrix of ones and 
\begin{equation*}
 I = \{p = (p_1, \dots, p_9) \in [0,1]^9 \colon \text{at least one } p_i \neq 0)\}
\end{equation*}
is a set of binary multi-indices. By the notation $\odot p_i$ in~\eqref{eq: MNIST Gram matrix} we mean the element-wise exponentiation of the matrices $G_i$, i.e., $G^{\odot 0}_i = \mathds{1}$ and $G^{\odot 1}_i = G_i$. The matrix $G$ is the sum over all possible $q$-combinations, $q \geq 1$, of Hadamard products of the matrices $G_1, \dots, G_9$ as depicted in the right part of Figure~\ref{fig: MNIST partition}. Each combination represents the application of the kernel corresponding to the feature map~\eqref{eq: MNIST feature map} on a subgroup of the different blocks. That is, we consider
\begin{equation*}
 \sum_{q =1}^9 \begin{pmatrix} 9 \\ q \end{pmatrix} = 2^9 -1 = 511
\end{equation*}
different kernel functions since the product of kernels again forms a kernel, see~\cite{STEINWART2008}. The same holds for the sum of kernels. Let us denote the kernel function corresponding to the Gram matrix $G$ in~\eqref{eq: MNIST Gram matrix} by $k$.

In~\cite{STOUDENMIRE2016}, the proposed value for the kernel parameter is $\kappa = \pi / 2$. However, as discussed in~\cite{KLUS2019}, the optimal choice can vary. Thus, we consider different values for the kernel parameter $\kappa$ as well as for the threshold $\varepsilon$ in Algorithm~\ref{alg: sample reduction}. Since it is reasonable to assume that the feature vectors corresponding to different categories are linearly independent, we split the training set into the different classes and apply Algorithm~\ref{alg: sample reduction} to each category separately. The number of samples extracted by our method is shown in Figure~\ref{fig: MNIST results}~(a). 

\begin{figure}[htbp]
\centering
\begin{subfigure}{0.45\textwidth}
    \centering
    \caption{}
    \includegraphics[width=7cm]{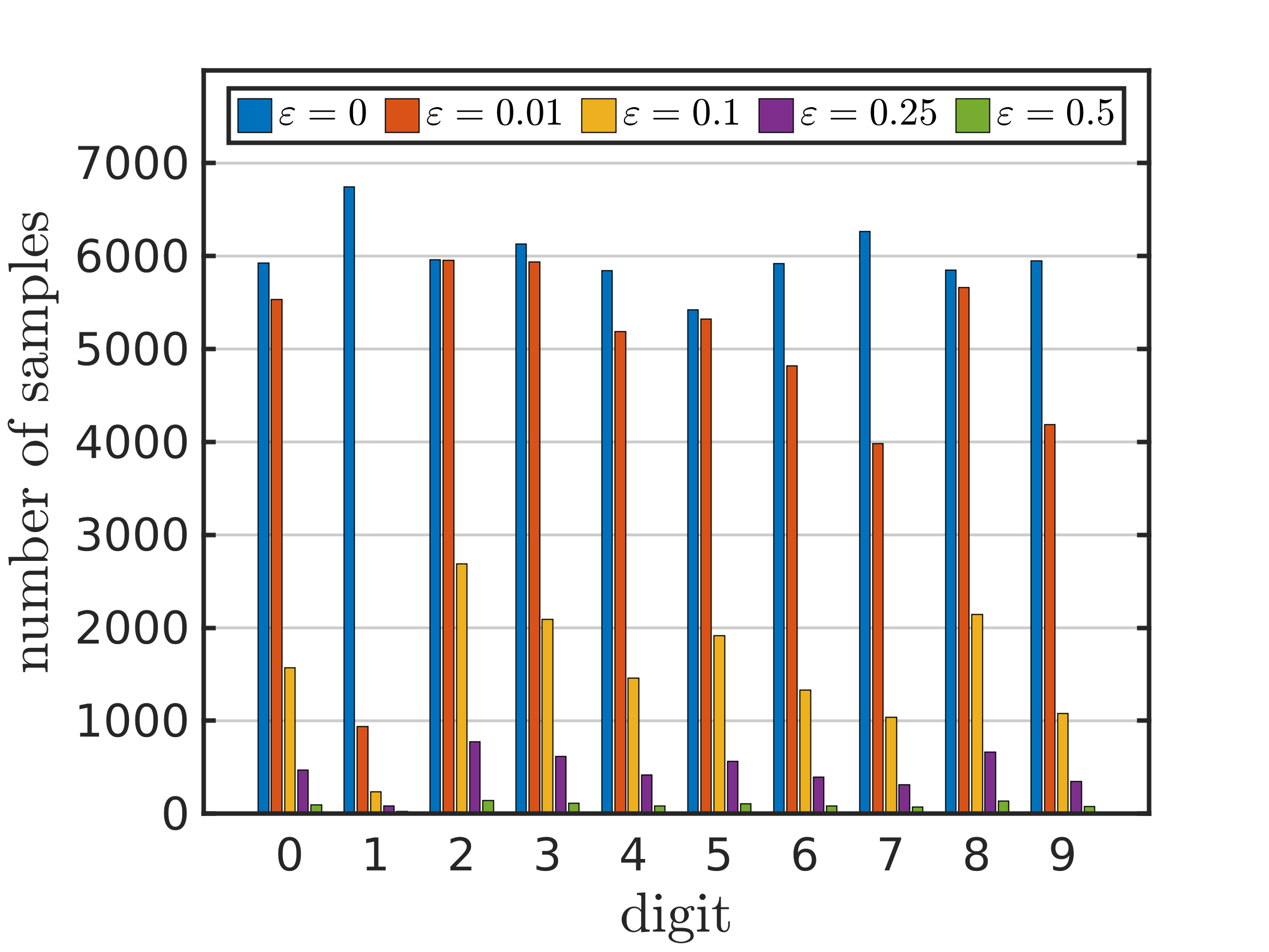}
\end{subfigure}%
\hspace*{0.75cm}
\begin{subfigure}{0.45\textwidth}
    \centering
    \caption{}
    \includegraphics[width=7cm]{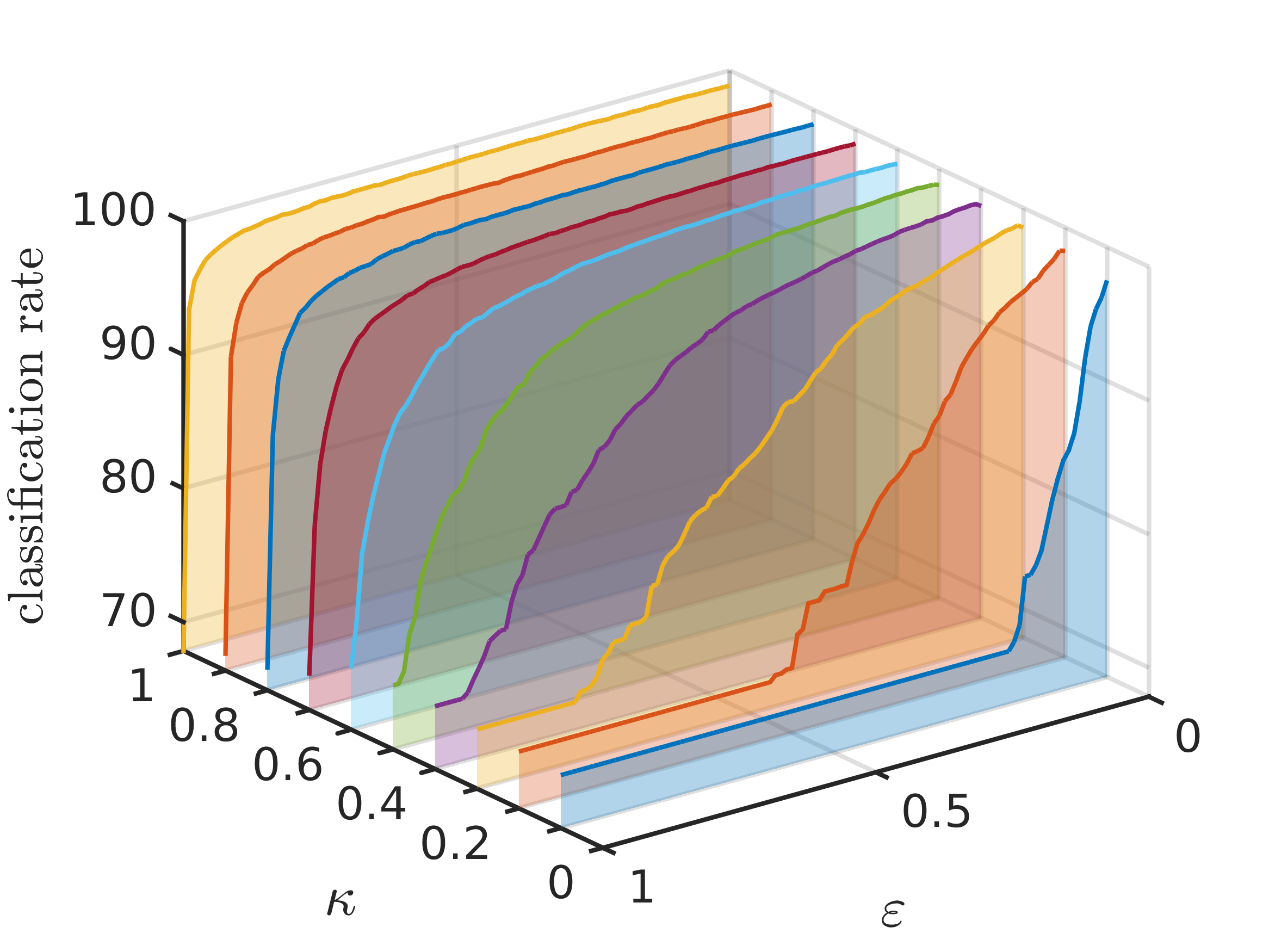}
\end{subfigure}%
\caption{Application of kFSA to the MNIST data set: (a) Number of extracted samples per digit for $\kappa=0.5$ and different thresholds. (b) Classification rates obtained on the test set as functions of $\varepsilon$ for different values of $\kappa$, both between $0$ and $1$. The intervals in which the classification rates are constant for small values of $\kappa$ indicate that only one sample for each category is extracted.}
\label{fig: MNIST results}
\end{figure}

The numbers of extracted samples for different thresholds show that the digit $1$ is subject to the least variation since even for thresholds close to zero only a relatively small number of samples remains. For instance, if we consider $\kappa=1$ and set $\varepsilon = 0.05$, nearly all samples from the training data set are extracted by Algorithm~\ref{alg: sample reduction} except for the samples corresponding to digit $1$ where only $1820$ out of $6742$ images are included in the set $\tilde{X}$. In contrast to that, the feature vectors of the samples corresponding to the digit $2$ show the highest degree of linear independence, i.e., a comparatively large number of samples is needed in order to reach the desired approximation errors~\eqref{eq: approximation error} among all respective feature vectors. Since the considered kernel is normalized, i.e., $k(x,x) = 1$ for all $x \in X$, the approximation errors are smaller than $1$ for any initial sample $x_0$~\eqref{eq: x_0} in a given category. Thus, only one sample is extracted per category if we choose a threshold of $\varepsilon=1$, whereas all given samples in the training set are taken into account if we set $\varepsilon=0$. In the latter case, the minimization problems~\eqref{eq: minimization problem 2} are ill-conditioned and additional regularization is necessary. Therefore, we apply \emph{kernel ridge regression}~\cite{MURPHY2012}, i.e., we solve the minimization problem 
\begin{equation*}
  \min_{\Theta \in \R^{d^prime \times m}} \lVert Y - \Theta \left(G_{X,X} + \gamma \mathrm{Id} \right) \rVert_F^2
\end{equation*}
with $\gamma=10^{-10}$. 

For the test phase, we solve the minimization problem~\eqref{eq: minimization problem 3} and construct the corresponding target function~\eqref{eq: minimization problem 3 - f}. Given a test vector $x$, the index of the largest entry of $f(x)$ (interpreted as a vector of likelihoods) then defines the predicted category. Figure~\ref{fig: MNIST results}~(b) shows the amount of correctly identified images in the test set for different kernel parameters and thresholds. As the kernel parameter approaches $1$, the classification rate tends to increase even for larger thresholds. On the one hand, we observe the effect that for a fixed threshold more and more samples are extracted with increasing kernel parameter. On the other hand, however, note that we achieve high classification rates using only a fraction of the training set instead the whole set. In fact, Table~\ref{tab: MNIST results} shows that in most cases the generalizability increases for some thresholds $\varepsilon>0$.

\begin{table}[h]
\renewcommand{\arraystretch}{1.3}
  \caption{Application of kFSA to the MNIST data set: For each choice of $\kappa$, we show the threshold $\varepsilon>0$ and the numbers of samples per digit corresponding to the highest classification rates (CR) obtained by applying Algorithm~\ref{alg: sample reduction}. Except for the digit $1$ -- where we get a classification rate of $97.53\%$ when using all $60000$ samples -- the results are better than the respective rates for $\varepsilon=0$.}
  \centering
  \begin{tabular}{c|c|cccccccccc|c}
    \hline
    \multirow{2}{*}{$\kappa$} & \multirow{2}{*}{$\varepsilon$} & \multicolumn{10}{c|}{number of samples per digit} & \multirow{2}{*}{CR}\\
    \cline{3-12}
     &  & $0$ & $1$ & $2$ & $3$ & $4$ & $5$ & $6$ & $7$ & $8$ & $9$ & \\
    \hline 
    
    $0.1$ & $0.01$ & $74$ & $41$ & $91$ & $81$ & $76$ & $82$ & $74$ & $69$ & $88$ & $73$ & $96.36\%$\\
    $0.2$ & $0.01$ & $510$ & $142$ & $751$ & $623$ & $473$ & $592$ & $443$ & $375$ & $669$ & $401$ & $98.39\%$\\
    
    $0.3$ & $0.01$ & $1892$ & $340$ & $3027$ & $2426$ & $1773$ & $2237$ & $1608$ & $1317$ & $2430$ & $1319$ & $98.82\%$\\
    
    $0.4$ & $0.02$ & $2763$ & $418$ & $4336$ & $3477$ & $2559$ & $3210$ & $2310$ & $1814$ & $3426$ & $1837$ & $98.95\%$\\
    
    $0.5$ & $0.04$ & $3301$ & $449$ & $4972$ & $4062$ & $3024$ & $3760$ & $2708$ & $2117$ & $3980$ & $2178$ & $98.99\%$\\
    
    $0.6$ & $0.07$ & $3747$ & $469$ & $5370$ & $4515$ & $3399$ & $4179$ & $3055$ & $2377$ & $4393$ & $2468$ & $99.01\%$\\
    
    $0.7$ & $0.16$ & $3137$ & $368$ & $4901$ & $3974$ & $2872$ & $3650$ & $2566$ & $1964$ & $3920$ & $2038$ & $98.98\%$\\
    
    $0.8$ & $0.19$ & $3923$ & $441$ & $5527$ & $4742$ & $3581$ & $4349$ & $3196$ & $2464$ & $4613$ & $2576$ & $98.99\%$\\
    
    $0.9$ & $0.19$ & $4951$ & $582$ & $5884$ & $5614$ & $4582$ & $5058$ & $4140$ & $3265$ & $5372$ & $3470$ & $98.98\%$\\
    
    $1.0$ & $0.27$ & $4927$ & $554$ & $5881$ & $5604$ & $4555$ & $5038$ & $4100$ & $3222$ & $5359$ & $3436$ & $98.98\%$\\
    
    \hline 
  \end{tabular}
  \label{tab: MNIST results}
\end{table}

For example, the highest classification rate of $99.01\%$ is obtained by only using approximately $57\%$ of the given training samples. Moreover, only $2035$ of the $60000$ samples are needed to achieve a classification rate larger than $98\%$ ($\kappa=0.6$, $\varepsilon=0.54$, CR$=98.03\%$). Thus, small numbers of samples, extracted by Algorithm~\ref{alg: sample reduction}, may be sufficient to obtain satisfying or even increased classification rates. The best result without applying kFSA, $98.93\%$, is obtained when setting $\kappa =0.7$.

\subsection{Fermi--Pasta--Ulam--Tsingou model}

As a second example for the application of kFSA, we consider the Fermi--Pasta--Ulam--Tsingou model, see \cite{FERMI1955}, which represents a vibrating string by a system of $d$ coupled oscillators fixed at the end points as shown in Figure~\ref{fig: fpu model}. 

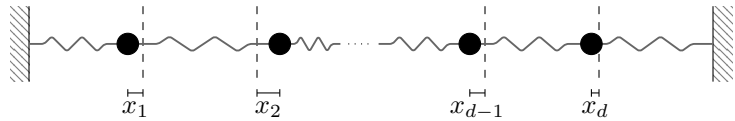
\begin{figure}[htbp]
\centering
\begin{tikzpicture}
\def\d{1.5}
\def\h{1}
\def\y{0}

\draw[pattern=north west lines, pattern color=black!50, draw=none, inner sep=0, outer sep=0] (5-3*\d, \y-0.5*\h) rectangle (5-3*\d-0.25,\y+0.5*\h);
\draw[pattern=north west lines, pattern color=black!50, draw=none, inner sep=0, outer sep=0] (5+3*\d, \y-0.5*\h) rectangle (5+3*\d+0.25,\y+0.5*\h);

\draw[] (5-3*\d,\y+0.5*\h) -- ++ (0,-\h);
\draw[dashed] (5-2*\d,\y+0.5*\h) -- ++ (0,-\h);
\draw[dashed] (5-1*\d,\y+0.5*\h) -- ++ (0,-\h);
\draw[dashed] (5+1*\d,\y+0.5*\h) -- ++ (0,-\h);
\draw[dashed] (5+2*\d,\y+0.5*\h) -- ++ (0,-\h);
\draw[] (5+3*\d,\y+0.5*\h) -- ++ (0,-\h);

\FPeval{\xa}{5-3*\d}
\FPeval{\xb}{5-2*\d+0.2-0.5}
\draw[rounded corners=1pt, draw=black!60, line width=0.025cm] (\xa,\y) -- ++ (0.166*\xb-0.166*\xa,0) -- ++ (0.0833*\xb-0.0833*\xa,0.1*\h) -- ++ (0.166*\xb-0.166*\xa,-0.2*\h) -- ++ (0.166*\xb-0.166*\xa,0.2*\h) -- ++ (0.166*\xb-0.166*\xa,-0.2*\h) -- ++ (0.0833*\xb-0.0833*\xa,0.1*\h) -- ++ (0.166*\xb-0.166*\xa,0);

\FPeval{\xa}{5-2*\d+0.2-0.3}
\FPeval{\xb}{5-1*\d-0.4+0.6}
\draw[rounded corners=1pt, draw=black!60, line width=0.025cm] (\xa,\y) -- ++ (0.166*\xb-0.166*\xa,0) -- ++ (0.0833*\xb-0.0833*\xa,0.1*\h) -- ++ (0.166*\xb-0.166*\xa,-0.2*\h) -- ++ (0.166*\xb-0.166*\xa,0.2*\h) -- ++ (0.166*\xb-0.166*\xa,-0.2*\h) -- ++ (0.0833*\xb-0.0833*\xa,0.1*\h) -- ++ (0.166*\xb-0.166*\xa,0);

\FPeval{\xa}{5-1*\d-0.4+0.8}
\FPeval{\xb}{5-1*\d-0.2+0.1+1.2}
\draw[rounded corners=1pt, draw=black!60, line width=0.025cm] (\xa,\y) -- ++ (0.166*\xb-0.166*\xa,0) -- ++ (0.0833*\xb-0.0833*\xa,0.1*\h) -- ++ (0.166*\xb-0.166*\xa,-0.2*\h) -- ++ (0.166*\xb-0.166*\xa,0.2*\h) -- ++ (0.166*\xb-0.166*\xa,-0.2*\h) -- ++ (0.0833*\xb-0.0833*\xa,0.1*\h) -- ++ (0.166*\xb-0.166*\xa,0);

\draw[dotted] (5-1*\d-0.2+0.1+1.2+0.1,\y) -- (5+1*\d-0.1-0.1-1.2-0.1,\y);

\FPeval{\xa}{5+1*\d-0.1-0.1-1.2}
\FPeval{\xb}{5+1*\d-0.2-0.1}
\draw[rounded corners=1pt, draw=black!60, line width=0.025cm] (\xa,\y) -- ++ (0.166*\xb-0.166*\xa,0) -- ++ (0.0833*\xb-0.0833*\xa,0.1*\h) -- ++ (0.166*\xb-0.166*\xa,-0.2*\h) -- ++ (0.166*\xb-0.166*\xa,0.2*\h) -- ++ (0.166*\xb-0.166*\xa,-0.2*\h) -- ++ (0.0833*\xb-0.0833*\xa,0.1*\h) -- ++ (0.166*\xb-0.166*\xa,0);

\FPeval{\xa}{5+1*\d-0.2+0.1}
\FPeval{\xb}{5+2*\d+0.3-0.5}
\draw[rounded corners=1pt, draw=black!60, line width=0.025cm] (\xa,\y) -- ++ (0.166*\xb-0.166*\xa,0) -- ++ (0.0833*\xb-0.0833*\xa,0.1*\h) -- ++ (0.166*\xb-0.166*\xa,-0.2*\h) -- ++ (0.166*\xb-0.166*\xa,0.2*\h) -- ++ (0.166*\xb-0.166*\xa,-0.2*\h) -- ++ (0.0833*\xb-0.0833*\xa,0.1*\h) -- ++ (0.166*\xb-0.166*\xa,0);

\FPeval{\xa}{5+2*\d+0.15-0.3}
\FPeval{\xb}{5+3*\d}
\draw[rounded corners=1pt, draw=black!60, line width=0.025cm] (\xa,\y) -- ++ (0.166*\xb-0.166*\xa,0) -- ++ (0.0833*\xb-0.0833*\xa,0.1*\h) -- ++ (0.166*\xb-0.166*\xa,-0.2*\h) -- ++ (0.166*\xb-0.166*\xa,0.2*\h) -- ++ (0.166*\xb-0.166*\xa,-0.2*\h) -- ++ (0.0833*\xb-0.0833*\xa,0.1*\h) -- ++ (0.166*\xb-0.166*\xa,0);

\draw[] (5-2*\d,\y-0.5*\h-0.15) -- (5-2*\d-0.2,\y-0.5*\h-0.15) node[below, midway] {$x_1$};
\draw[] (5-2*\d,\y-0.5*\h-0.15+0.05) -- ++ (0,-0.1);
\draw[] (5-2*\d-0.2,\y-0.5*\h-0.15+0.05) -- ++ (0,-0.1);

\draw[] (5-1*\d,\y-0.5*\h-0.15) -- (5-1*\d+0.3,\y-0.5*\h-0.15) node[below, midway] {$x_2$};
\draw[] (5-1*\d,\y-0.5*\h-0.15+0.05) -- ++ (0,-0.1);
\draw[] (5-1*\d+0.3,\y-0.5*\h-0.15+0.05) -- ++ (0,-0.1);

\draw[] (5+1*\d,\y-0.5*\h-0.15) -- (5+1*\d-0.2,\y-0.5*\h-0.15) node[below, midway] {$x_{d-1}$};
\draw[] (5+1*\d,\y-0.5*\h-0.15+0.05) -- ++ (0,-0.1);
\draw[] (5+1*\d-0.2,\y-0.5*\h-0.15+0.05) -- ++ (0,-0.1);

\draw[] (5+2*\d,\y-0.5*\h-0.15) -- (5+2*\d-0.1,\y-0.5*\h-0.15) node[below, midway] {$x_d$};
\draw[] (5+2*\d,\y-0.5*\h-0.15+0.05) -- ++ (0,-0.1);
\draw[] (5+2*\d-0.1,\y-0.5*\h-0.15+0.05) -- ++ (0,-0.1);

\node[shape=circle,inner sep=0.1cm, outer sep=0, draw, fill=black] at (5-2*\d-0.2,\y) {};
\node[shape=circle,inner sep=0.1cm, outer sep=0, draw, fill=black] at (5-1*\d+0.3,\y) {};
\node[shape=circle,inner sep=0.1cm, outer sep=0, draw, fill=black] at (5+1*\d-0.2,\y) {};
\node[shape=circle,inner sep=0.1cm, outer sep=0, draw, fill=black] at (5+2*\d-0.1,\y) {};

\end{tikzpicture}
\caption{Fermi--Pasta--Ulam--Tsingou model: Representation of a vibrating string by a set of masses coupled by springs.}
\label{fig: fpu model}
\end{figure}

The governing equations of this dynamical system are given by second-order differential equations of the form
\begin{equation}\label{eq: FPU}
    \ddot{x}_i = (x_{i+1} - 2 x_i + x_{i-1}) + \beta \left((x_{i+1} - x_i)^3 - (x_i - x_{i-1})^3 \right),
\end{equation}
for $i=1, \dots, d$, where we assume cubic coupling terms between the variables $x_i$ representing the displacements from the stationary positions of the oscillators, see Figure~\ref{fig: fpu model}. We set $x_0 = x_{d+1}=0$ since the end points of the chain are supposed to be fixed. The parameter $\beta \in \R$ defines the coupling strength. In \cite{GELSS2019}, we applied MANDy -- the tensor-based version of SINDy~\cite{BRUNTON2016} -- to randomly generated training data in order to recover the governing equations \eqref{eq: FPU}. Here, we want to consider the same problem, but apply kernel-based regression as described in Section~\ref{sec: data-reduced regression} using a polynomial kernel. That is, we consider the kernel function
\begin{equation}\label{eq: FPU kernel}
  k(x, x^\prime) = (\kappa + x^\top x^\prime)^q,
\end{equation}
with $\kappa=1$ and $q=3$. We construct the data matrices $X, Y \in \R^{d \times m}$ by varying the number of oscillators and generating random displacements in $[-0.1, 0.1]$ for each oscillator as well as the corresponding derivatives according to~\eqref{eq: FPU} with $\beta = 0.7$. Due to the random sampling, we expect the (implicitly given) feature space to be sampled sufficiently so that the feature vectors span the whole feature space. Thus, we try to use kFSA to determine the feature space dimension $n$. Figure~\ref{fig: FPU results} shows the number of indices that are extracted by Algorithm~\ref{alg: sample reduction} when applied to randomly generated training data with different dimensions $d$ and a threshold of $10^{-10}$.

\begin{figure}[htbp]
\centering
\begin{subfigure}{0.45\textwidth}
    \centering
    \caption{}
    \includegraphics[height=5.25cm]{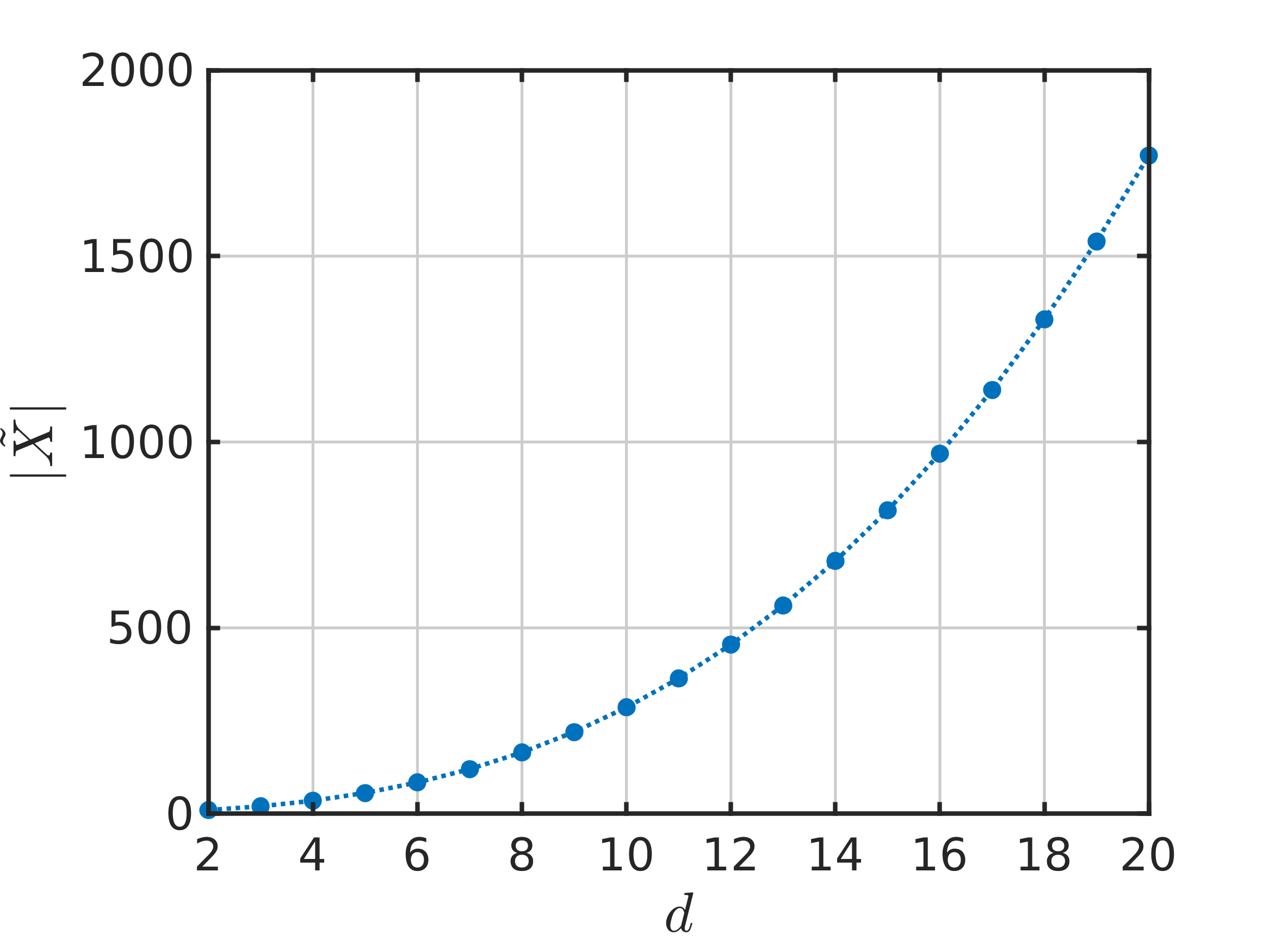}
\end{subfigure}%
\hspace*{0.052\textwidth}
\begin{subfigure}{0.45\textwidth}
    \centering
    \caption{}
    \includegraphics[height=5.25cm]{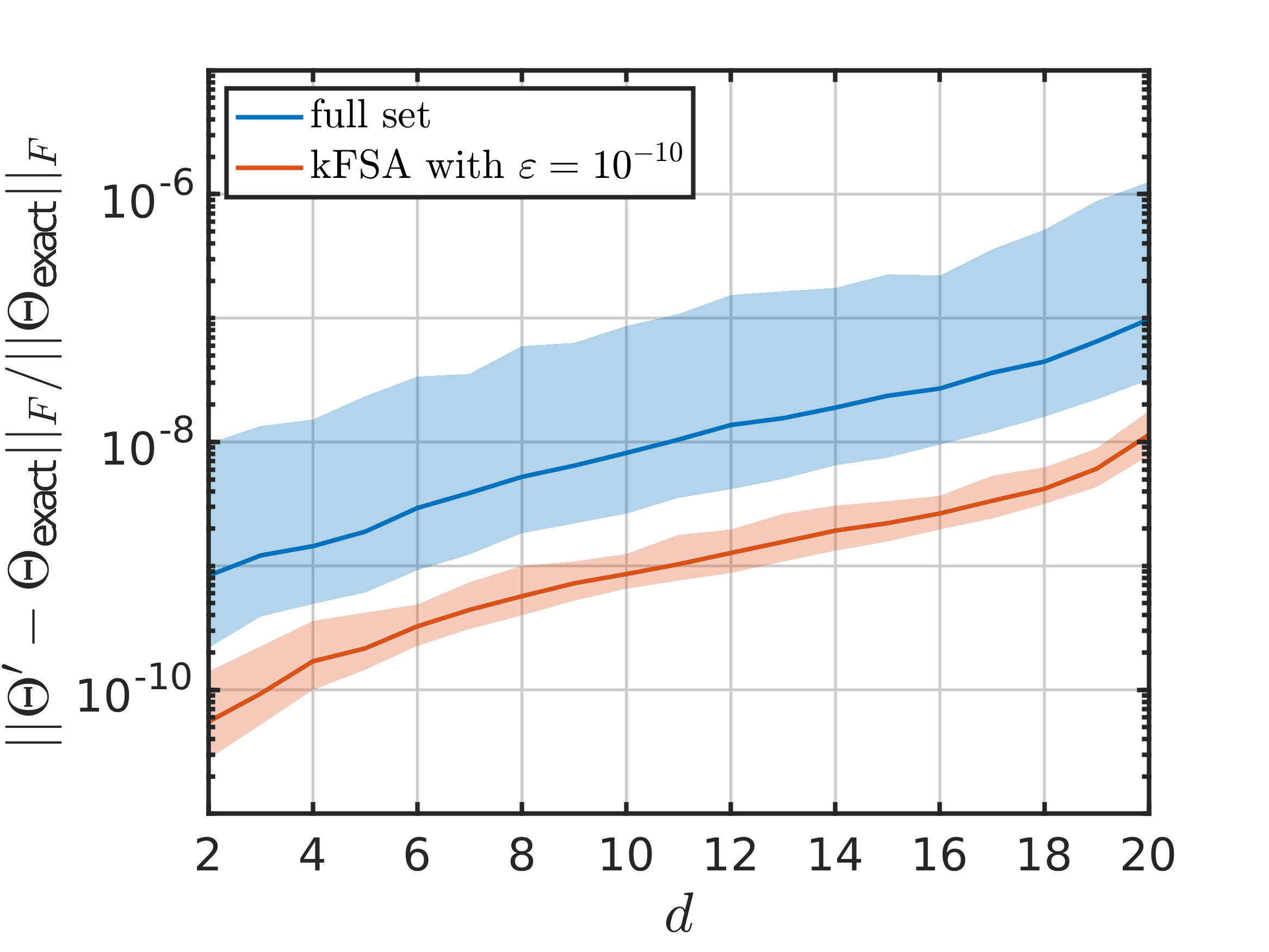}
\end{subfigure}%
\caption{Application of kFSA to the Fermi--Pasta--Ulam--Tsingou problem: (a) Number of samples extracted by Algorithm~\ref{alg: sample reduction} when applied to $2000$ randomly generated $d$-dimensional displacement vectors. (b) Median of the relative errors of the approximate coefficient matrices over $d$ for the full set as well as the reduced set with $1000$ realizations for the former and $100$ realizations for latter case. The semi-transparent areas comprise the $5$th to the $95$th percentile of the respective results.}
\label{fig: FPU results}
\end{figure}

As we can see in Figure~\ref{fig: FPU results}~(a), our method extracts less than $2000$ samples for each dimension $d$ up to 20. Note that for different realizations of $X$, the set $\tilde{X}$ varies but the number of samples $|\tilde{X}|$ remains unchanged. The cubic dependency of the number of important samples on $d$ can be explained as follows: For the polynomial kernel, we can directly write down an explicit expression of a feature map. As shown in, e.g., \cite{ZAKI2014}, the kernel~\eqref{eq: FPU kernel} can be expressed as
\begin{equation}\label{eq: FPU feature map}
  k(x, x^\prime) = \sum_{|p| = q} \left( \sqrt{a_{p}} \prod_{i=1}^d x_i^{p_i}\right) \left( \sqrt{a_{p}} \prod_{i=1}^d {x^\prime_i}^{p_i}\right) = \Psi(x)^\top \Psi(x^\prime),
\end{equation}
where $p = (p_0, p_1, \dots, p_d) \in \N_0^{d+1}$ is a multi-index and $|p| = p_0 + \dots + p_d$. The prefactors  $a_{p}$ are given in terms of multinomial coefficients:
\begin{equation}\label{eq: FPU prefractors}
  a_{p} = \begin{pmatrix} q \\ p \end{pmatrix} \alpha^{p_0} = \frac{q!}{p_0 ! p_1 ! \cdots p_d !} \alpha^{p_0}.
\end{equation}
Thus, the feature vector $\Psi(x)$ in \eqref{eq: FPU feature map} contains all monomials of order up to and including three in the variables $x_i$, scaled by the square roots of the prefactors defined in \eqref{eq: FPU prefractors}. The dimension of the feature space is then given by the number of all possible vectors $p$ with $| p | = q$, which is
\begin{equation*}
 n = \begin{pmatrix} d+q \\ q \end{pmatrix}.
\end{equation*}
For $q=3$, this reduces to
\begin{equation*}
 n = \frac{1}{6} (d+1)(d+2)(d+3),
\end{equation*}
which are the exact same numbers of samples extracted by our method for different dimensions $d$. That is, Algorithm~\ref{alg: sample reduction} detects the feature space dimension by choosing the correct amount of feature vectors to form a basis, given in form of column subsets $\tilde{X} \subset X$. As described in Section~\ref{sec: data-reduced regression}, we can use the transformed data matrix $\Psi_{\tilde{X}}$ to reconstruct the coefficient matrix~\eqref{eq: coefficient matrix 1} representing the solution of~\eqref{eq: minimization problem 1}, i.e., to recover the governing equations. For a given set $\tilde{X} \subseteq X$, the regression function is defined as
\begin{equation*}
  f(x) = \Theta \Psi_{\tilde{X}}^\top \Psi(x) = \underbrace{\Theta \Psi_{\tilde{X}}^\top D}_{= \Theta^\prime} \underbrace{D^{-1} \Psi(x)}_{=\Psi^\prime(x)},
\end{equation*}
where $\Theta$ is the solution of \eqref{eq: minimization problem 3} and $D$ is a diagonal matrix containing the corresponding prefactors given in~\eqref{eq: FPU prefractors}. $\Psi^\prime(x)$ then only contains monomials with prefactor $1$ and the exact coefficient matrix defined by~\eqref{eq: FPU} becomes
\begin{equation*}
 \Theta_\text{exact} = 
 \kbordermatrix{
 ~ & 1 & x_1 & x_1^2 &  x_1^3 & x_2 & x_1 x_2 & x_1^2 x_2 &        & x_{d-2} x_d^2 & x_{d-1} x_d^2 & x_d^3\\
 ~ & 0 & -2  &     0 &   -1.4 &   1 &       0 &       2.1 &    ~   &             0 &             0 &    0~\\
 ~ & 0 &  1  &     0 &    0.7 &  -2 &       0 &      -2.1 & \cdots &             0 &             0 &    0~\\
 ~ & 0 &  0  &     0 &      0 &   1 &       0 &         0 &    ~   &             0 &             0 &    0~\\
 ~ & ~ &  ~  &     ~ & \vdots &   ~ &       ~ &         ~ & \ddots &             ~ &        \vdots &    ~~\\
 ~ & 0 &  0  &     0 &      0 &   0 &       0 &         0 &    ~   &             0 &             0 &    0~\\
 ~ & 0 &  0  &     0 &      0 &   0 &       0 &         0 & \cdots &             0 &          -2.1 &  0.7~\\
 ~ & 0 &  0  &     0 &      0 &   0 &       0 &         0 &    ~   &             0 &           2.1 & -1.4~\\
 }.
\end{equation*}
For the considered values of $d$ and randomly generated training sets, Figure~\ref{fig: FPU results}~(b) shows that the relative approximation error in the Frobenius norm between $\Theta^\prime$ and the exact coefficient matrix is more stable and approximately one order of magnitude lower for the reduced set than for the full set. Particularly, we observe more extreme outliers in the results for $\varepsilon=0$, where we do not apply any regularization technique, than for $\varepsilon=10^{-10}$. This indicates that the set generated by kFSA can be used to improve system identification performance since, in addition to reduced numerical costs, we also benefit from regularization effects.

\subsection{CalCOFI}\label{sec: CalCOFI}

The third example is an analysis of time series data taken from the \emph{CalCOFI} (California Cooperative Oceanic Fisheries Investigations) data base~\cite{CALCOFI2020}. CalCOFI is a partnership of several institutions founded in 1949 to investigate the sardine collapse off California and today provides one of the longest-running time series that exist. These data -- containing different seawater measurements such as depth, pressure, and temperature -- are the most reliable oceanographic time series available. For the demonstration of our method, we randomly extract a training and a test set from CalCOFI's final upcast CTD (Conductivity, Temperature, Depth) data.\!\footnote{\url{https://calcofi.org/data/ctd/455-ctd-data-files.html}} The training set contains $25000$ measurement vectors while the test set contains $5000$ samples. Various relationships in the data have been examined using supervised learning techniques, see for instance~\cite{ALIN2012, KIM2015, SAKAMOTO2017}. In this work, the aim is to find a regression function that predicts the dissolved oxygen in the seawater as a function of the depth, pressure, temperature, and salinity.

Without any knowledge of the relationship between the input and output variables, we choose Gaussian kernels for the kernel-based regression, i.e.,
\begin{equation*}
  k(x, x^\prime) = \exp \left( -\kappa \lVert x-x^\prime \rVert_2^2 \right)
\end{equation*}
with $x, x^\prime \in \mathcal{S} \subset \R^4$ and kernel parameter $\kappa >0$. Since the input variables live on different scales, we normalize these measurements in order to avoid the effect of larger-scale variables dominating the others. That is, given a training data matrix ${X = [x^{(1)}, \dots, x^{(m)}] \in \R^{4 \times m}}$ we apply \emph{min-max normalization} in the form of
\begin{equation*}
   \hat{X}_{i, j} = \hat{x}^{(j)}_i = \frac{x^{(j)}_i - l_i}{u_i - l_i} \in [0,1]
\end{equation*}
for $i=1, \dots, 4$ and $j = 1,\dots, m$, where 
\begin{equation*}
  l_i = \min_{j \in \{1, \dots, m\}} x^{(j)}_i = \min X_{i,:} \qquad \text{and} \qquad u_i = \max_{j \in \{1, \dots, m\}} x^{(j)}_i = \max X_{i,:}.
\end{equation*}
In the test phase, we then normalize any given vector by using the same constants before applying the regression function $f \colon \R^4 \to \R$ to predict the oxygen concentration.

\begin{figure}[h]
\centering
\begin{subfigure}{0.45\textwidth}
    \centering
    \caption{}
    \includegraphics[height=5.25cm]{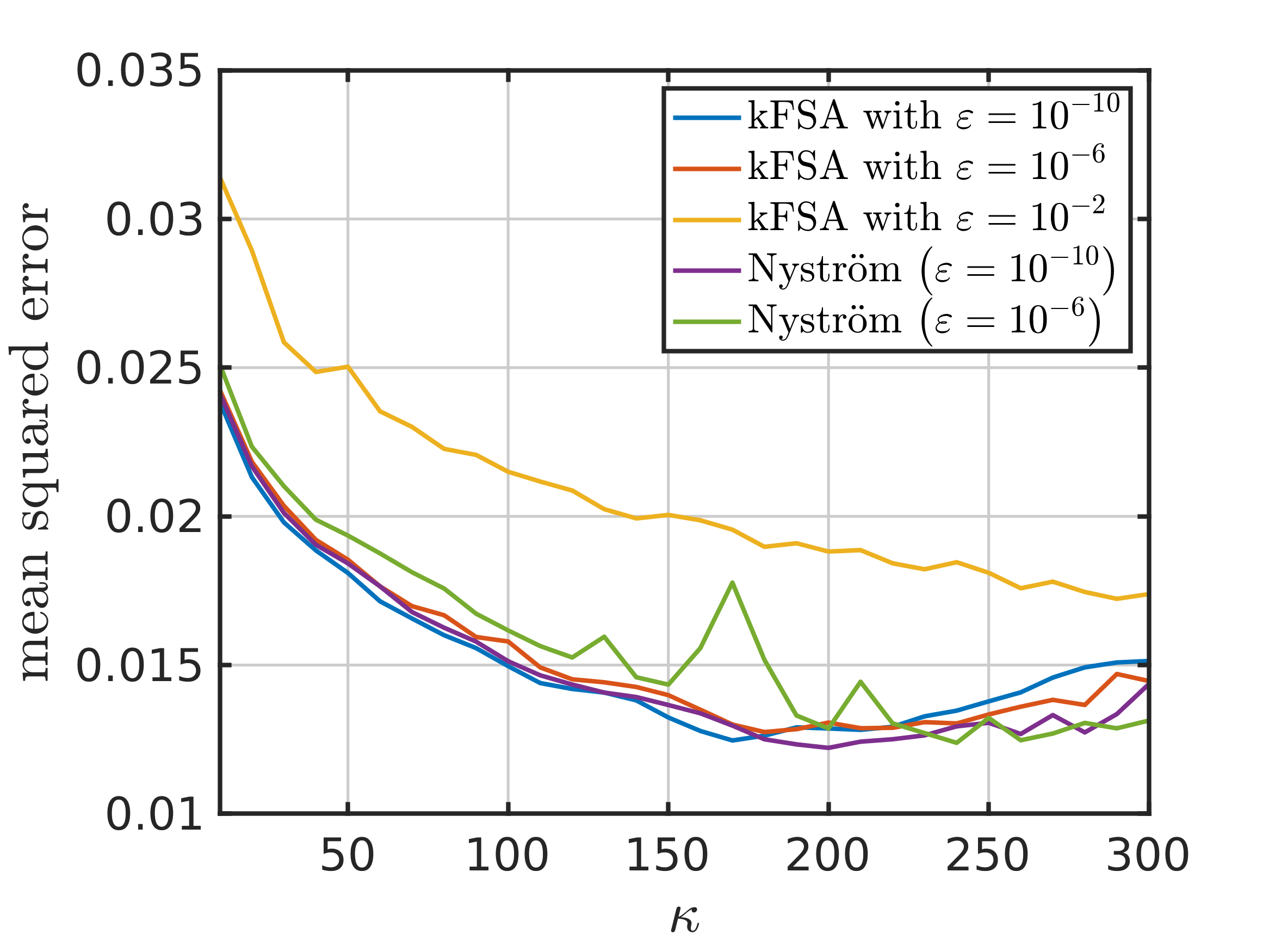}
\end{subfigure}%
\hspace*{0.052\textwidth}
\begin{subfigure}{0.45\textwidth}
    \centering
    \caption{}
    \includegraphics[height=5.25cm]{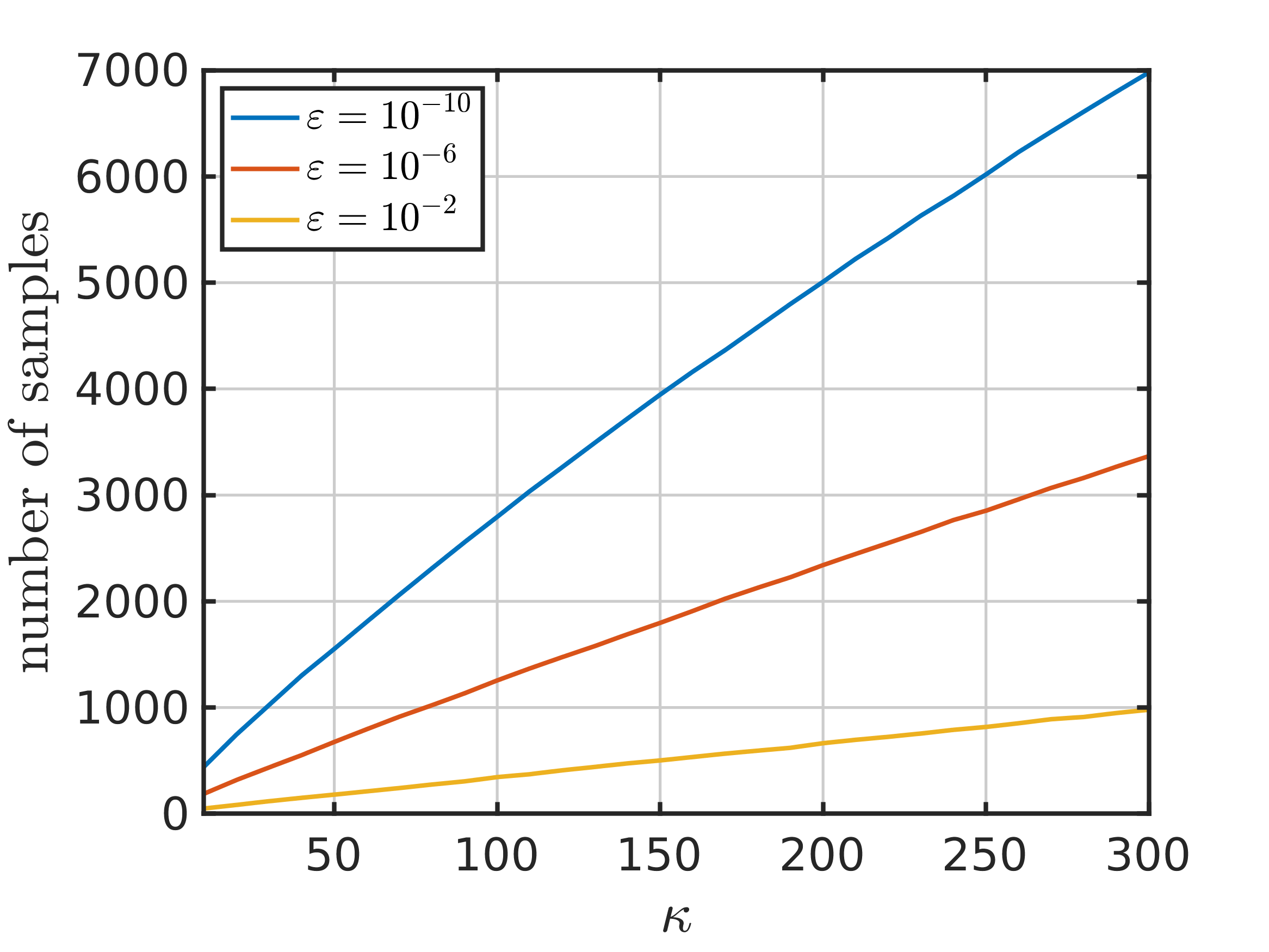}
\end{subfigure}\\%
\begin{subfigure}{\textwidth}
    \centering
    \caption{}
    \includegraphics[width=14.5cm]{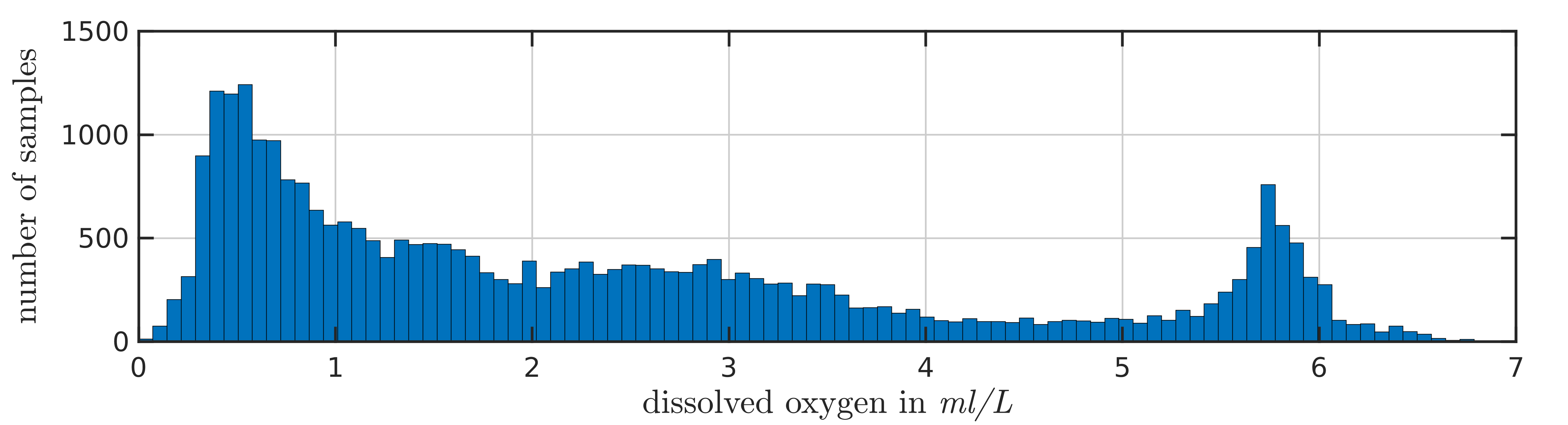}
\end{subfigure}%
\caption{Application of kFSA method to the CalCOFI data set: (a) Mean squared errors in the predicted dissolved oxygen values for different thresholds and kernel parameters. (b) Number of samples extracted by Algorithm~\ref{alg: sample reduction} for different thresholds and kernel parameters. (c) Histogram of the output data (including training and test set).}
\label{fig: CalCOFI results}
\end{figure}

The feature space corresponding to the Gaussian kernel is infinite-dimensional and a Gram matrix is always non-singular for mutually different samples. That is, given $X = [x^{(1)}, \dots, x^{(m)}]$ with ${x^{(i)} \neq x^{(j)}}$ for $i \neq j$, the matrix $G_{X,X}$ has rank $m$. In particular, all submatrices $G_{\tilde{X}, X}$ would therefore have rank $\tilde{m} = |\tilde{X}|$. However, our experiments show that the computed Gram matrices corresponding to subsets $\tilde{X} \subset X$ extracted by kFSA are still ill-conditioned. Thus, we use additional regularization in order to construct well-posed systems of linear equations. For the extraction of relevant samples, we will employ Algorithm~\ref{alg: sample reduction} with different thresholds $\varepsilon$ as well as, for comparison, the Nyström method described in Section~\ref{sec: analytical formulation}. Suppose we have found a subset $\tilde{X} \subset X$, we then solve the minimization problem~\eqref{eq: minimization problem 3} by constructing a regularized system of linear equations
\begin{equation*}
  \tilde{\Theta} \left( G_{\tilde{X}, X} G_{X, \tilde{X}} + \gamma \mathrm{Id} \right) = Y G_{X, \tilde{X}},
\end{equation*}
where we used the normal equation of~\eqref{eq: minimization problem 3}. Here, the (row) vector $Y \in \R^m$ contains the oxygen levels corresponding to the samples given in the matrix $X$. The regression function is then defined as
\begin{equation*}
  f(x) = \tilde{\Theta} G_{\tilde{X}, \hat{x}},
\end{equation*}
see~\eqref{eq: minimization problem 3 - f}, where $\hat{x}$ denotes the normalized version of $x$ as described above. The results in form of the mean squared errors in the output variables for $\gamma = 10^{-10}$ and different values for $\kappa$ as well as $\varepsilon$ are shown in Figure~\ref{fig: CalCOFI results}~(a).

The mean squared errors lie between $0.0125$ and $0.0314$ for any choice of the parameters $\kappa$ and $\varepsilon$. As we can see in Figure~\ref{fig: CalCOFI results}~(c), the calculated errors are negligibly small in comparison to the majority of the oxygen measurements in the CalCOFI data set. The results for $\varepsilon= 10^{-10}$ and $\varepsilon= 10^{-6}$ are quite similar, which indicates that even for smaller thresholds we will not get significantly lower errors. This assumption is also supported by the mean squared errors corresponding to the Nyström method proposed in~\cite{MUSCO2017}, which is provably accurate for any kernel. In order to compare both methods, we used the Nyström method\,\footnote{\url{https://github.com/cnmusco/recursive-nystrom}} to extract subsets of the same size as obtained by applying kFSA with $\varepsilon= 10^{-10}$ and $\varepsilon= 10^{-6}$, respectively, see Figure~\ref{fig: CalCOFI results}~(b). If we allow the same amount of samples as determined by Algorithm~\ref{alg: sample reduction} with $\varepsilon=10^{-10}$, the results of the Nyström method are comparable to the results of Algorithm~\ref{alg: sample reduction} for thresholds $\varepsilon=10^{-10}$ and $\varepsilon=10^{-6}$. If we only allow as many samples as determined by Algorithm~\ref{alg: sample reduction} with $\varepsilon=10^{-6}$, we see that the errors are slightly larger and erratic. For smaller numbers of samples, the results obtained by applying the Nyström method start to oscillate -- even worse than they already do for $\varepsilon=10^{-6}$. That is, for small thresholds Algorithm~\ref{alg: sample reduction} extracts subsets $\tilde{X}$ which are more suitable for a numerically stable regression than the subsets extracted by the Nyström method. As shown in Figure~\ref{fig: CalCOFI results}~(b), the number of extracted samples depends sublinearly on the kernel parameter $\kappa$. For the considered parameter combinations, that number is always smaller than $7000$, which corresponds to less than $30\%$ of the given training data. 

Note that further improvements in terms of computational complexity of kFSA will be subject to future research. At this point, the runtime of kFSA is not competitive. For instance, the Nyström method we used for comparison requires only $O(m M^2)$ operations instead of $O(m^2 + m M^2)$, where $M$ is the number of extracted samples, see Lemma~\ref{lemma: kFSA complexity}.

\section{Conclusion \& outlook}\label{sec: Conclusion}

We have proposed a data-reduction approach for supervised learning tasks. The resulting method -- called \emph{kernel-based feature space approximation} or, in short, kFSA -- extracts relevant samples from a given data set whose corresponding feature vectors can be used as a (approximate) basis of the feature space spanned by the transformations of all samples up to an arbitrary accuracy. The aim is reduce the storage consumption and computational complexity as well as to achieve a regularization effect that improves the generalizability of the learned target functions. The introduced method is purely written in terms of kernel functions, which allows us to avoid explicit representations of high-dimensional feature maps and even to consider infinite-dimensional feature spaces. 

The performance of kFSA was demonstrated using examples of classification and regression problems from different scientific areas such as image recognition, system identification, and time-series analysis. In our experiments, we observed that the computational effort may be reduced significantly while obtaining even better results than when considering all training data in the learning phase. Moreover, we showed that kFSA is able to detect the correct dimension of a feature space that is only implicitly given by a kernel. We also found cases where sample subsets extracted by kFSA can lead to numerically more stable results than subsets found by a state-of-the-art implementation of the Nyström method.

Future research will include the application of kFSA to a broader spectrum of problems in the field of data-driven analysis of dynamical systems, e.g., transfer operator approximation, model reduction, and system identification. Other open issues are algorithmic improvements for the bottom-up construction of relevant subsets, imposing additional or alternative constraints such as sparsity of the coefficient matrices, and the artificial generation of input-output pairs to further reduce the number of needed basis vectors in the feature space.

\section*{Acknowledgments}

This research has been funded by Deutsche Forschungsgemeinschaft (DFG) through grant CRC 1114 \emph{``Scaling Cascades in Complex Systems''} (project ID: 235221301, project B06).

{\small{}\bibliographystyle{unsrturl}
\bibliography{references}
}{\small\par}

\end{document}